\documentclass{article}
\usepackage{microtype}
\usepackage{graphicx}
\usepackage{subfigure}
\usepackage{booktabs} 

\usepackage{hyperref}

\usepackage{amsmath}
\usepackage{amssymb}
\usepackage{mathtools}
\usepackage{amsthm}

\usepackage[capitalize,noabbrev]{cleveref}

\theoremstyle{plain}
\newtheorem{theorem}{Theorem}[section]

\newtheorem{lemma}[theorem]{Lemma}

\theoremstyle{definition}
\newtheorem{definition}[theorem]{Definition}

\theoremstyle{remark}

\usepackage[textsize=tiny]{todonotes}

\usepackage{url}            
\usepackage{booktabs}       
\usepackage{amsfonts}       
\usepackage{nicefrac}       
\usepackage{microtype}      
\usepackage{xcolor}         

\usepackage{graphicx}
\usepackage{subfigure}
\usepackage{booktabs} 
\usepackage{amsmath}
\usepackage{amssymb}
\usepackage{mathtools}
\usepackage{amsthm}
\usepackage{makecell}
\usepackage{color}
\usepackage{array}

\usepackage{tcolorbox}
\usepackage{multirow}
\usepackage{nccmath}
\usepackage{booktabs}
\usepackage{wrapfig}

\newcommand{\M}{WaveGC}
\usepackage{caption}
\usepackage{bbm}
\usepackage{bm}
\newtheorem{Definition}{\textbf{Definition}}

\usepackage[accepted]{icml2025}

\icmltitlerunning{A General Graph Spectral Wavelet Convolution via Chebyshev Order Decomposition}
\begin{document}

\twocolumn[
\icmltitle{A General Graph Spectral Wavelet Convolution via \\Chebyshev Order Decomposition}



\icmlsetsymbol{equal}{*}

\begin{icmlauthorlist}
\icmlauthor{Nian Liu}{yyy}
\icmlauthor{Xiaoxin He}{yyy}
\icmlauthor{Thomas Laurent}{comp}
\icmlauthor{Francesco Di Giovanni}{sch}
\icmlauthor{Michael M. Bronstein}{sch,qqq}
\icmlauthor{Xavier Bresson}{yyy}
\end{icmlauthorlist}

\icmlaffiliation{yyy}{National University of Singapore}
\icmlaffiliation{comp}{Loyola Marymount University}
\icmlaffiliation{sch}{University of Oxford}
\icmlaffiliation{qqq}{AITHYRA, Austria}

\icmlcorrespondingauthor{Nian Liu}{nianliu@comp.nus.edu.sg}

\icmlkeywords{Machine Learning, ICML}

\vskip 0.3in
]



\printAffiliationsAndNotice{} 

\begin{abstract}
Spectral graph convolution, an important tool of data filtering on graphs, relies on two essential decisions: selecting spectral bases for signal transformation and parameterizing the kernel for frequency analysis. While recent techniques mainly focus on standard Fourier transform and vector-valued spectral functions, they fall short in flexibility to model signal distributions over large spatial ranges, and capacity of spectral function.
In this paper, we present a novel wavelet-based graph convolution network, namely \M, which integrates multi-resolution spectral bases and a matrix-valued filter kernel. Theoretically, we establish that \M\ can effectively capture and decouple short-range and long-range information, providing superior filtering flexibility, surpassing existing graph wavelet neural networks. 
To instantiate \M, we introduce a novel technique for learning general graph wavelets by separately combining odd and even terms of Chebyshev polynomials. This approach strictly satisfies wavelet  admissibility criteria. Our numerical experiments showcase the consistent improvements in both short-range and long-range tasks. This underscores the effectiveness of the proposed model in handling different scenarios.  Our code is available at \url{https://github.com/liun-online/WaveGC}.
\end{abstract}
\section{Introduction}
Spectral graph theory (SGT)~\cite{sgt}, which enables analysis and learning on graph data, has firmly established itself as a pivotal methodology in graph machine learning. A significant milestone in SGT is the generalization of the convolution operation to graphs, as convolution for grid-structured data, i.e. sequences and images, has demonstrated remarkable success \cite{lecun1998gradient,hinton2012deep, krizhevsky2012imagenet}.
Significant research interests in graph convolution revolve around two key factors: (1) \textit{designing diverse bases for spectral transform}, and (2) \textit{parameterizing powerful graph kernel}. For (1), the commonly used graph Fourier basis, consisting of the eigenvectors of the graph Laplacian~\cite{shuman2013emerging}, stands as a prevalent choice. However, graph wavelets~\cite{gw} offer enhanced flexibility by constructing adaptable bases. For (2), classic approaches involve diagonalizing the kernel with fully free parameters~\cite{bruna2013spectral} or employing various polynomial approximations such as Chebyshev~\cite{chebnet} and Cayley~\cite{levie2018cayleynets} polynomials. Additionally, convolution with a tensor-valued kernel serves as the spectral function of Transformer~\cite{va_trans} under the shift-invariant condition~\cite{fno,afno}.

Despite the existence of techniques in each aspect, the integration of these two lines into a unified framework remains challenging, impeding the full potential of graph convolution. In an effort to unravel this challenge, we introduce a novel operation — \textbf{Wave}let-based \textbf{G}raph \textbf{C}onvolution (\M), which seamlessly incorporates both spectral basis and kernel considerations.
In terms of spectral basis design, \M\ is built upon graph wavelets, allowing it to capture information across the entire graph through a multi-resolution approach from highly adaptive construction of multiple graph wavelet bases. For filter parameterization, we opt for a matrix-valued spectral kernel with weight-sharing. The matrix-valued kernel offers greater flexibility to filter wavelet signals, thanks to its larger parameter space.


To comprehensively explore \M, we theoretically analyse and assess its information-capturing capabilities. In contrast to the K-hop basic message-passing framework, \M\ is demonstrated to exhibit both significantly larger and smaller receptive fields concurrently, achieved through the manipulation of scales. Previous graph wavelet theory~\cite{gw} only verifies the localization in small scale limit. Instead, our proof is complete as it covers both extremely small and large scales from the perspective of information mixing~\cite{how}.
Moreover, our proof also implies that \M\ is capable of simultaneously capturing both short-range and long-range information for each node, which facilitate global node interaction. 

To implement \M, a critical step lies in constructing graph wavelet bases that satisfy two fundamental criteria: (1) meeting the wavelet admissibility criteria~\cite{cw} and (2) showing adaptability to different graphs. Existing designs of graph wavelets face limitations, with some falling short in ensuring the criteria~\cite{gwnn, xu2022graph}, while others having fixed wavelet forms, lacking adaptability~\cite{ufg, cho2023multi}.
To address these limitations, we propose an innovative and general implementation of graph wavelets. Our solution involves \textit{approximating scaling function basis and multiple wavelet bases using odd and even terms of Chebyshev polynomials, respectively}. This approach is inspired by our observation that, after a certain transformation, 
even terms of Chebyshev polynomials strictly satisfy the  admissibility  criteria, while odd terms supplement direct current signals.
Through the combination of these terms via learnable coefficients, we aim to theoretically approximate scaling function and multiple wavelets with arbitrary complexity and flexibility. Our contributions are:
\begin{itemize}
\setlength{\itemsep}{0pt}
\setlength{\parsep}{0pt}
\setlength{\parskip}{0pt}
    \item 
    We derive a new wavelet-based graph convolution (\M), which integrates multi-resolution bases and matrix-valued kernel, enhancing spectral convolution on large spatial ranges.
    \item We theoretically prove that \M\ can capture and distinguish the information from short and long ranges, surpassing conventional graph wavelet techniques.
    \item We pioneer an implementation of learnable graph wavelets, employing odd terms and even terms of Chebyshev polynomials individually. This implementation strictly satisfies the wavelet admissibility criteria.
    \item Our approach consistently outperforms baseline methods on both short-range and long-range tasks, achieving up to 15.7\% improvement on VOC dataset.    
\end{itemize}

\section{Preliminaries}
\label{pre}
An undirected graph can be presented as $\mathcal{G}=(\mathcal{V}, E)$, where $\mathcal{V}$ is the set of $N$ nodes and $E\subseteq \mathcal{V}\times\mathcal{V}$ is the set of edges. The adjacency matrix of this graph is $\bm{A}\in\{0,1\}^{N\times N}$, where $\bm{A}_{ij}\in\{0,1\}$ denotes the relation between nodes $i$ and $j$ in $\mathcal{V}$. The degree matrix is $\bm{D}=\textrm{diag}(d_1,\dots.d_N)\in\mathbb{R}^{N\times N}$, where $d_i=\sum_{j\in \mathcal{V}}\bm{A}_{ij}$ is the degree of node $i \in \mathcal{V}$. The node feature matrix is $\bm{X}=[x_1, x_2,\dots, x_N]\in\mathbb{R}^{N\times d_0}$, where $x_i$ is a $d_0$ dimensional feature vector of node $i \in \mathcal{V}$. Let $\hat{\bm{A}}=\bm{D}^{-\frac{1}{2}}\bm{A}\bm{D}^{-\frac{1}{2}}$ be the symmetric normalized adjacency matrix, then $\hat{\mathcal{\bm{L}}}=\bm{I_N}-\hat{\bm{A}}=\bm{D}^{-\frac{1}{2}}(\bm{D}-\bm{A})\bm{D}^{-\frac{1}{2}}$ is the symmetric normalized graph Laplacian. With eigen-decomposition, $\hat{\mathcal{\bm{L}}}=\bm{U}\bm{\Lambda}\bm{U}^\top$, where $\bm{\Lambda}=\textrm{diag}(\lambda_1,\dots,\lambda_N)\in\mathbb{R}^{N\times N}, \lambda_i\in[0,2]$ and $\bm{U}=[\bm{u_1^\top},\dots,\bm{u_N^\top]}\in\mathbb{R}^{N\times N}$ are the eigenvalues and eigenvectors of $\hat{\mathcal{\bm{L}}}$, respectively. Given a signal  $f\in\mathbb{R}^N$ on  $\mathcal{G}$, the graph Fourier transform~\citep{shuman2013emerging} is defined as $\hat{f}=\bm{U}^\top f\in\mathbb{R}^N$, and its inverse is $f=\bm{U}\hat{f}\in\mathbb{R}^N$.

\textbf{Spectral graph wavelet transform (SGWT)}.~\citet{gw} redefine the wavelet basis~\citep{cw} on vertices in the spectral graph domain. Specifically, the SGWT is composed of three components: (1)~\textit{Unit wavelet basis}, denoted as $\Psi$ such that $\Psi=g(\hat{\mathcal{\bm{L}}})=\bm{U}g(\bm{\Lambda})\bm{U}^\top$, where $g$ acts as a band-pass filter $g: \mathbb{R}^+\rightarrow\mathbb{R}^+$ meeting the following \textit{wavelet admissibility criteria}~\citep{cw}:
\begin{equation}
    \label{condition}
    \mathcal{C}_\Psi = \int_{-\infty}^\infty \frac{|g(\lambda)|^2}{|\lambda|} d\lambda < \infty.
\end{equation}
To meet this requirement, $g(\lambda=0)=0$ and $\lim_{\lambda\rightarrow\infty}g(\lambda)=0$ are two essential prerequisites. (2)~\textit{Spatial scales}, a series of positive real values $\{s_j\}$ where distinct values of $s_j$ with $\Psi_{s_j}=\bm{U}g(s_j\bm{\Lambda})\bm{U}^\top$ can control different size of neighbors. (3)~\textit{Scaling function basis}, denoted as $\Phi$ such that $\Phi=\bm{U}h(\lambda)\bm{U}^\top$. Here, the function of $h: \mathbb{R}^+\rightarrow\mathbb{R}^+$ is to supplement direct current (DC) signals at $\lambda=0$, which is omitted by all wavelets $g(s_j\lambda)$ since $g(0)=0$. Next, given a signal $f\in\mathbb{R}^N$, the formal SGWT~\citep{gw} is:
\begin{equation}
\label{sgwt}
W_f(s_j)=\Psi_{s_j}f=\bm{U}g(s_j\bm{\Lambda})\bm{U}^\top f\in\mathbb{R}^N,
\end{equation}
where $W_f(s_j)$ is the wavelet coefficients of $f$ under scale $s_j$. Similarly, scaling function coefficients are given by $S_f=\Phi f=\bm{U}h(\bm{\Lambda})\bm{U}^\top f\in\mathbb{R}^N$. Let $G(\lambda)=h(\lambda)^2+\sum_jg(s_j\lambda)^2$, then if $G(\lambda)\equiv 1,\ \forall \lambda\in\bm{\Lambda}$, the constructed graph wavelets are known as \textit{tight frames}, which guarantee energy conservation of the given signal between the original and the transformed domains~\citep{tight}. More spectral graph wavelets are introduced in Appendix~\ref{related work}. 
\section{From Graph Convolution to Graph Wavelets}
\label{methoddddd}
Spectral graph convolution is a fundamental operation in the field of graph signal processing~\citep{shuman2013emerging}. Specifically, given a signal matrix (or node features) $\bm{X}\in\mathbb{R}^{N\times d}$ on graph $\mathcal{G}$, the spectral filtering of this signal is defined with a kernel $\bm{\kappa}\in\mathbb{R}^{N\times N}$ by the convolution theorem~\citep{convolution}:
\begin{equation}
\label{convolution definition}
    \bm{\kappa}*_\mathcal{G}\bm{X}=\mathcal{F}^{-1}(\mathcal{F}(\bm{\kappa})\cdot\mathcal{F}(\bm{X})) \in\mathbb{R}^{N\times d},
\end{equation}
where $\cdot$ is the matrix multiplication operator, $\mathcal{F}(\cdot)$ and $\mathcal{F}^{-1}(\cdot)$ are the spectral transform (e.g., graph Fourier transform~\citep{bruna2013spectral}) and corresponding inverse transform, respectively. To implement a spectral convolution, two critical choices must be considered in Eq.~(\ref{convolution definition}): 1) the selection of the transform $\mathcal{F}$ and 2) the parameterization of the kernel $\bm{\kappa}$.

\subsection{General spectral wavelet via Chebyshev decomposition}
For the selection of the spectral transform $\mathcal{F}$ and its inverse $\mathcal{F}^{-1}$, it can be tailored to the specific nature of data. For set data, the Dirac Delta function~\citep{oppenheim1997signals} is employed, while the fast Fourier Transform (FFT) proves efficient for both sequences~\citep{fno} and grids~\citep{afno}. In the context of graphs, the Fourier transform ($\mathcal{F}\rightarrow\bm{U}^\top$) emerges as one classical candidate. However, some inherent flaws limit the capacity of Fourier bases. (1) Standard graph Fourier bases, represented by one fixed matrix $\bm{U}^\top$, maintain a constant resolution and fixed frequency modes. (2) Fourier bases lack the adaptability to be further optimized according to different datasets and tasks. Therefore, \textit{multiple resolution} and \textit{adaptability} are two prerequisites for the design of an advanced base.

Notably, wavelet base is able to conform the above two demands, and hence offers enhanced filtering compared to Fourier base. For the resolution, the use of different scales $s_j$ allows wavelet to analyze detailed components of a signal at different granularities. More importantly, due to its strong spatial localization~\citep{gw}, each wavelet corresponds to a signal diffused away from a central node~\citep{gwnn}. Therefore, these scales also control varying receptive fields in spatial space, which enables the simultaneous fusion of short- and long-range information. For the adaptability, graph wavelets offer the flexibility to adjust the shapes of wavelets and scaling function. These components can be collaboratively optimized for the alignment of basis characteristics with different datasets, potentially enhancing generalization performance.

Next, we need to determine the form of the scaling function basis $\Phi=\bm{U}h(\bm{\Lambda})\bm{U}^\top$, the unit wavelet basis $\Psi=\bm{U}g(\bm{\Lambda})\bm{U}^\top$, and the scales ${s_j}$. The forms of $h$ and $g$ are expected to be powerful enough and easily available. Concurrently, $g$ should strictly satisfy the wavelet admissibility criteria, i.e., Eq.~(\ref{condition}), and $h$ should complementally provide DC signals. To achieve this target, we \textit{separately introduce odd terms and even terms from Chebyshev polynomials~\citep{gw}} into the approximation of $h$ and $g$. Please recall that the Chebyshev polynomial $T_k(y)$ of order $k$ may be computed by the stable recurrence relation $T_k(y)=2yT_{k-1}(y)-T_{k-2}(y)$ with $T_0=1$ and $T_1=y$. After the following transform, we surprisingly observe that these transformed terms match all above expectations:
\begin{equation}
    \label{transform}
    T_k(y) \rightarrow 1/2\cdot(-T_k(y-1)+1).
\end{equation}

To give a more intuitive illustration, we present the spectra of first six Chebyshev polynomials before and after the transform in Fig.~\ref{model} (b), where the set of odd and even terms after the transform are denoted as $\{T^o_i\}$ and $\{T^e_i\}$, respectively. From the figure, $g(\lambda=0)\equiv0$ for all $\{T^e_i\}$, and $h(\lambda=0)\equiv1$ for all $\{T^o_i\}$. Consequentially, $\{T^e_i\}$ and $\{T^o_i\}$ strictly meet the criteria and naturally serve as the basis of unit wavelet and scaling function. Moreover, not only can we easily get each Chebyshev term via iteration, but the constructed wavelet owns arbitrarily complex waveform because of the combination of as many terms as needed. Given $\{T^e_i\}$ and $\{T^o_i\}$, all we need to do is just to learn the coefficients to form the corresponding $g(\lambda)$ and $h(\lambda)$:
\begin{equation}
\label{asfasfas}
\begin{aligned}
    g(\bm{\Lambda})=\sum\limits_i^\rho a_iT^e_i(\bm{\Lambda}) \in\mathbb{R}^{N\times N},\\
    h(\bm{\Lambda})=\sum\limits_i^\rho b_iT^o_i(\bm{\Lambda}) \in\mathbb{R}^{N\times N},
\end{aligned}
\end{equation}
where $\rho=K/2$ ($K$ is the total number of truncated Chebyshev terms), $\tilde{\bm{a}}=(a_1, a_2,\dots, a_\rho)\in\mathbb{R}^{1\times \rho}$ and $\tilde{\bm{b}}=(b_1, b_2,\dots, b_\rho)\in\mathbb{R}^{1\times \rho}$ represent two learnable coefficient vectors as follows:
\begin{equation}
\label{gagagag}
    \tilde{\bm{a}} = \textrm{Mean}(\bm{W_a}\hat{\bm{Z}}+\bm{b_a}),\quad \tilde{\bm{b}} = \textrm{Mean}(\bm{W_b}\hat{\bm{Z}}+\bm{b_b}),
\end{equation}
where $\{\bm{W_a}, \bm{W_b}\}\in\mathbb{R}^{d\times \rho}$ and $\{\bm{b_a}, \bm{b_b}\}\in\mathbb{R}^{1\times \rho}$ are learnable parameters, and $\hat{\bm{Z}}$ is the eigenvalue embedding composed by the module in~\citep{specformer}. Further details can be found in Appendix~\ref{ee_}. 
Also, we can learn the scales $\tilde{\bm{s}}=(s_1, s_2,\dots, s_J)$ in the same way:
\begin{equation}
\label{zxvzxv}
    \tilde{\bm{s}} = \sigma(\textrm{Mean}(\bm{W_s}\hat{\bm{Z}}+\bm{b_s}))\cdot\overline{\bm{s}}\in\mathbb{R}^{1\times J},
\end{equation}
where $\sigma$ is sigmoid function, $\bm{W_s}\in\mathbb{R}^{d\times J}$ and $\bm{b_s}\in\mathbb{R}^{1\times J}$ are learnable parameters, and $\overline{\bm{s}}=(\overline{s_1}, \overline{s_2},\dots, \overline{s_J})$ is a pre-defined vector to control the size of $\tilde{\bm{s}}$. 

Based on our construction, $g(\lambda)$ is a strict band-pass filter in [0, 2], while $s$ can scale its shape in $g(s\lambda)$. Specifically, $s<1$ "stretches" the shape of $g(\lambda)$, and $s>1$ "squeezes" its shape (Please refer to Fig.~\ref{scalesssss}). To maintain the same spectral interval [0, 2], we truncate $g(s\lambda)$ within the intersection of $\lambda\in$ [0, 2] and $\lambda\in$ [0, 2/s].

\begin{figure*}[t]
  \centering
  \includegraphics[scale=0.17]{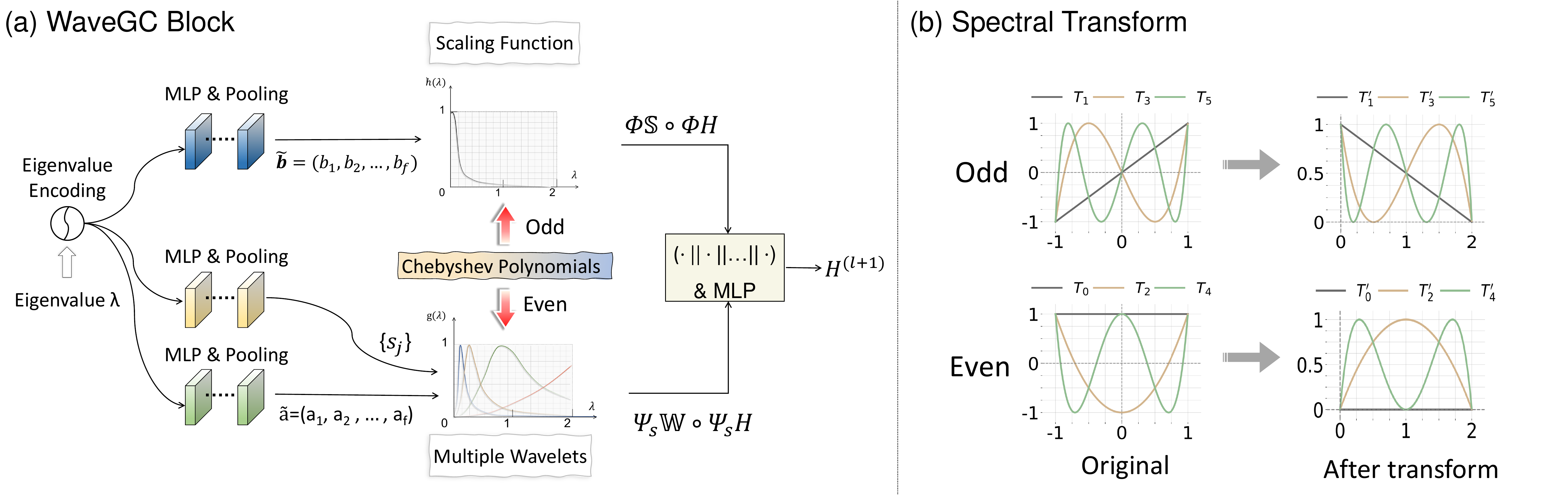}
    \vspace{-0.25cm}
  \caption{(a) Overview of our proposed \M\ technique. (b) Illustration of Chebyshev polynomials before and after the given transform, from [-1, 1] to [0, 2]. In this representation, we distinguish odd and even terms, presenting only the first three terms for each.}
  \label{model}
  \vskip -0.15in
\end{figure*}

\subsection{Matrix-valued kernel via weight sharing}
Next, we consider the parametrization of the convolutional kernel $\mathcal{F}(\bm{\kappa})$. In the spectral domain, each Fourier mode typically corresponds to a global frequency pattern, either low- or high-frequency. Consequently, in Fourier-based approaches, it is common to apply a vector-valued kernel over the diagonalized graph Laplacian spectrum, denoted as $\text{diag}(\theta_\lambda)$~\citep{bruna2013spectral, chebnet, levie2018cayleynets}, which effectively scales these global frequency components. However, this strategy becomes unsuitable after applying a wavelet transform. Unlike Fourier bases, wavelet coefficients encode localized, node-specific patterns that may capture short- or long-range interactions, but not global frequency modes. As a result, a different parametrization scheme, tailored to the localized nature of wavelet representations, is required.

Along another line of research, Fourier Neural Operator (FNO)~\citep{fno} models the convolution kernel as a fully learnable tensor $\mathbb{M}\in\mathbb{R}^{N\times d\times d}$, where $N$ is the number of frequency modes, and $d$ is the feature dimension. This tensor-valued kernel offers two notable advantages. First, although FNO was originally introduced in the context of the Fourier transform, the kernel $\mathbb{M}$ is inherently independent of graph spectrum, and is thus amenable to generalization across other transforms~\citep{tripura2023wavelet}. Second, in contrast to vector-valued kernels, the matrix-valued formulation provides a significantly larger number of learnable parameters, thereby increasing its expressivity and capacity to adapt to complex patterns. Experimental results presented in Section~\ref{matrix-value} empirically demonstrate that the matrix-valued kernel outperforms its vector-valued counterpart in the context of filtering wavelet-transformed signals.

In this paper, we adopt the tensor $\mathbb{M}$ for the convolution kernel. The standard parameter count for $\mathbb{M}$ is $N\times d\times d$. This can lead to a substantial number of parameters, especially for large-scale graphs with high $N$, increasing the risk of overfitting. To mitigate this while preserving model expressivity, we introduce a parameter-sharing strategy across all frequency modes by employing a single MLP. This approach reduces the number of learnable parameters from $N\times d\times d$ (tensor) to $d\times d$ (matrix). Accordingly, the convolution operation in Eq.~\eqref{convolution definition} simplifies to $\mathbb{M}*_\mathcal{G}\bm{X}=\mathcal{F}^{-1}\mathbb{M}\circ\mathcal{F}(\bm{X})=\mathcal{F}^{-1}(\textrm{MLP}(\mathcal{F}(\bm{X})))$, where $\circ$ is the composition between two functions. An alternative method is presented in AFNO~\citep{afno}, introducing a similar technique that offers improved efficiency but with a more intricate design.

\subsection{WAVELET-BASED GRAPH CONVOLUTION}
Until now, we have elaborated the proposed advancements on kernel and bases, and now discuss how to integrate these two aspects. Provided that we have $J$ wavelet $\{\Psi_{s_j}\}_{j=1}^J$ and one scaling function $\Phi$ constructed via the above Chebyshev decomposition, $\mathcal{F}: \mathbb{R}^{N\times d}\rightarrow \mathbb{R}^{N(J+1)\times d}$ in Eq.~(\ref{convolution definition}) is the stack of transforms from each component:
\begin{equation}
\begin{aligned}
\label{f}
    \mathcal{F}(\bm{H}^{(l)})&=\bm{T}\bm{H}^{(l)}=((\Phi\bm{H}^{(l)})^\top||\\ &(\Psi_{s_1}\bm{H}^{(l)})^\top||...||(\Psi_{s_J}\bm{H}^{(l)})^\top)^\top\in\mathbb{R}^{N(J+1)\times d},
\end{aligned}
\end{equation}
where $\bm{T}=(\Phi^\top|| \Psi_{s_1}^\top||...||\Psi_{s_J}^\top)^\top$ is the overall transform and $||$ means concatenation, $\bm{H}^{(l)}$ is the node embedding matrix at layer $l$. Next, we check if the inverse $\mathcal{F}^{-1}$ exists. Considering $\bm{T}$ is not a square matrix, $\mathcal{F}^{-1}$ should be its pseudo-inverse as $(\bm{T}^\top\bm{T})^{-1}\bm{T}^\top$, where $\bm{T}^\top\bm{T}=\Phi\Phi^\top+\sum_{j=1}^J \Psi_{s_j}\Psi_{s_j}^\top=\bm{U}[h(\lambda)^2+\sum_{j=1}^Jg(s_j\lambda)^2]\bm{U}^\top$. Ideally, if $\bm{T}$ is imposed as \textbf{tight frames}, then $h(\lambda)^2+\sum_{j=1}^Jg(s_j\lambda)^2=\bm{I}$~\citep{leonardi2013tight}, and $\bm{T}^\top\bm{T}=\bm{U}\bm{I}\bm{U}^\top=\bm{I}$. In this case, $\mathcal{F}^{-1}=(\bm{T}^\top\bm{T})^{-1}\bm{T}^\top=\bm{T}^\top$, and Eq.~(\ref{convolution definition}) becomes:
\begin{equation}
\begin{aligned}
\label{simplification}
\bm{H}^{(l+1)}&=\bm{T}^\top\mathbb{M}\circ\bm{T}\bm{H}^{(l)}\\
&=\Phi\mathbb{S}\circ\Phi\bm{H}^{(l)}+\sum_{j=1}^J\Psi_{s_j}\mathbb{W}_j\circ\Psi_{s_j}\bm{H}^{(l)}\in\mathbb{R}^{N\times d},
\end{aligned}
\end{equation}
where we separate $\mathbb{M}$ into $\mathbb{S}$ and $\{\mathbb{W}\}_{j=0}^J$ as scaling kernel and different wavelet kernels.

\begin{table*}[t]
  \caption{Comparison between spectral graph convolution and \M.}
  \label{comparison}
  \vskip 0.05in
  \centering
  \small
  \setlength{\tabcolsep}{20pt}{
  \renewcommand\arraystretch{1.3}
  \begin{tabular}{ccc}
    \toprule[1.2pt]
    & Spectral Graph Convolution &\M \\
    \hline
    Kernel & $\text{diag}(\theta_\lambda)$: Diagonal matrix & $\mathbb{S}$ / $\mathbb{W}$: Full matrix\\
    \hline
    Bases & $\textbf{U}^\top$: Fourier basis & $\Phi$ / $\Psi_s$: Scaling / Wavelet basis\\
    \hline
    Convolution & $\textbf{U}\text{diag}(\theta_\lambda)\textbf{U}^\top \textbf{X}$ & $\Phi \mathbb{S}\circ\Phi \textbf{X}\ /\ \Psi_s\mathbb{W}\circ\Psi_s \textbf{X}$\\
  \toprule[1.2pt]
  \end{tabular}
  }
  \vskip -0.1in
\end{table*}

\textbf{How to guarantee tight frames?} From above derivations, \textit{tight frames} is a key for the simplification of inverse $\mathcal{F}^{-1}$ in Eq.~(\ref{simplification}). This can be guaranteed by $l_2$ norm on the above constructed wavelets and scaling function. 
For each eigenvalue $\lambda_i\in\bm{\Lambda}$, we have $v_i^2=h(\lambda_i)^2+\sum_{j=1}^Jg(s_j\lambda_i)^2$, $\tilde{h}(\lambda_i)=h(\lambda_i)/v$, $\tilde{g_i}(s_j\lambda_i)=g(s_j\lambda_i)/v$. Then, $G(\bm{\Lambda})=\tilde{h}(\bm{\Lambda})^2+\sum_j\tilde{g}(s_j\bm{\Lambda})^2=\bm{I}$ forms tight frames (Section~\ref{pre}). Thus, while the pseudo-inverse must theoretically exist, we can circumvent the necessity of explicitly calculating the pseudo-inverse. 

Resembling the multi-head attention~\citep{va_trans}, we treat each wavelet transform as a “wavelet head”, and concatenate them rather than sum them to get $\bm{H}^{(l+1)}\in\mathbb{R}^{N\times d}$:
\begin{equation}
\begin{aligned}
    \label{final_1}
    \bm{H}^{(l+1)}&=\sigma\left(\left[\Phi\mathbb{S}\circ\Phi\bm{H}^{(l)}||\Psi_{s_1}\mathbb{W}_1\circ\Psi_{s_1}\bm{H}^{(l)}||\right.\right.\\
    &\left.\left.\dots||\Psi_{s_J}\mathbb{W}_J\circ\Psi_{s_J}\bm{H}^{(l)}\right]\cdot\bm{W}\right),
\end{aligned}
\end{equation}
where an outermost $\textrm{MLP}$ increases the flexibility. 
Fig.~\ref{model} (a) presents the whole framework of our wavelet-based graph convolution, or WaveGC. For a better understanding, we compare spectral graph convolution and \M\ in Table.~\ref{comparison}, where \M\ contains only one wavelet for simplicity. Based on the differences shown in the table, \M\ endows spectral graph convolution with the beneficial inductive bias of long-range dependency.



\section{Theoretical Properties of \M}
Traditionally, wavelet is notable for its diverse receptive fields because of varying scales~\citep{cw}. For graph wavelet,~\citet{gw} were the first to prove the localization when scale $s\rightarrow0$, but did not discuss the long-range case when $s\rightarrow\infty$. We further augment this discussion and demonstrate the effectiveness of the proposed \M\ in capturing both short- and long-range information. Intuitively, a model's ability to integrate global information enables the reception and mixing of messages from distant nodes. Conversely, a model with a limited receptive field can only effectively mix local messages. Hence, assessing the degree of information `mixing' becomes a key property. For this reason, we focus on the concept of \textit{maximal mixing}:
\begin{definition}
\textbf{(Maximal mixing)}~\citep{how}.
    \textit{For a twice differentiable graph-function $y_G$ of node features $\bm{x}_i$, the maximal mixing induced by $y_G$ among the features $\bm{x}_a$ and $\bm{x}_b$ with nodes $a, b$ is}
    \begin{equation}
    \label{definition}
        \text{mix}_{y_G}(a, b)=\max\limits_{\bm{x}_i}\max\limits_{1\leq\alpha, \beta\leq d}\left|\frac{\partial^2y_G(\bm{X})}{\partial x_a^\alpha\partial x_b^\beta}\right|.
    \end{equation}
\end{definition}
This definition is established in the context of graph-level task, and $y_G$ is the final output of an end-to-end framework, comprising the primary model and a readout function (e.g., \textrm{mean}, \textrm{max}) applied over the last layer. $\alpha$ and $\beta$ represent two entries of the $d$-dimensional features $\bm{x}_a$ and $\bm{x}_b$.

Next, we employ the concept of `maximal mixing' on the \M. For simplicity, we only take one wavelet basis $\Psi_{s}$ for analysis. 
The capacity of $\Psi_{s}$ on mixing information depends on two factors, i.e. $K$-order Chebyshev term and scale $s$. For a fair discussion on the effect of $s$ on message passing, we compare $\sigma(\Psi_{s}HW)$ and K-order message passing with the form of $\sigma(\sum_{j=0}^K\tau_jA^jHW)$, $\tau_j \in[0, 1]$:


\begin{theorem}[\textbf{Short-range and long-range receptive fields}]
\label{cc}
    Given a large even number $K>0$ and two random nodes $a$ and $b$, if the depths $m_\Psi$ and $m_A$ are necessary for $\sigma(\Psi_{s}HW)$ and $\sigma(\sum_{j=0}^K\tau_jA^jHW)$ to induce the same amount of mixing $\text{mix}_{y_G}(b, a)$, then the lower bounds of $m_\Psi$ and $m_A$, i.e. $L_{m_\Psi}$ and $L_{m_A}$, approximately satisfy the following relation when scale $s\rightarrow0$:
    \begin{equation}
    \label{zero}
        L_{m_\Psi} \approx \frac{P}{K}L_{m_A}+\frac{2|E|}{K\sqrt{d_ad_b}}\frac{\text{mix}_{y_G}(b, a)}{\gamma}\cdot\frac{1}{(\alpha^2s^{2K})^{m_\Psi}}.
    \end{equation}
    Or, if $s\rightarrow\infty$, the relation becomes:
    \begin{equation}
    \label{infty}
        L_{m_\Psi} \approx \frac{P}{K}L_{m_A}-\frac{2|E|}{K(K+1)^{2m_A} {\tau_P}^{2m_A}\sqrt{d_ad_b}}\frac{\text{mix}_{y_G}(b, a)}{\gamma},
    \end{equation}
    where $P<K$ and $(\tau_PA^P)_{ba}=\max\{(\tau_mA^m)_{ba}\}_{m=0}^K$. $d_a$ and $d_b$ are degrees of two nodes, and $\alpha=\frac{C\cdot2^K(K+1)}{K!}$. $\gamma=\sqrt{\frac{d_{max}}{d_{min}}}$, where $d_{max}/d_{min}$ is the maximum / minimum degree in the graph.
\end{theorem}

The proof is provided in Appendix~\ref{proof_cc}. In Eq.~(\ref{zero}), since the second term on the right-hand side is large ($s\rightarrow0$), it required $\Psi_{s}$ to propagate more layers to mix the nodes. Conversely, if $s\rightarrow\infty$ (Eq.~(\ref{infty})), $\Psi_{s}$ will achieve the same degree of node mixing as $K$-hop message passing but with less propagation. Moreover, the greater the "mixing" $\text{mix}_{y_G}(b, a)$ is required between nodes, the fewer number of layers $L_{m_\Psi}$ is needed compared to $L_{m_A}$. To conclude, $\Psi_{s}$ presents the short- and long-range characteristics of \M\ on message passing, while these characteristics do not derive from the order $K$ of Chebyshev polynomials but from the scale $s$ exclusively.

\begin{table*}[t]
\scriptsize
  \caption{Qualified results on short-range tasks compared to baselines. \textbf{Bold}: Best, \underline{Underline}: Runner-up, OOM: Out-of-memory. All results are reproduced based on source codes.}
  \label{node_com}
  \vskip 0.05in
  \centering
  \setlength{\tabcolsep}{6mm}{
  \renewcommand\arraystretch{1.1}
  \begin{tabular}{cccccc}
    \toprule[1.2pt]
    \multirow{2}*{Model} & CS & Photo & Computer& CoraFull &ogbn-arxiv  \\
    \cmidrule(lr){2-2}\cmidrule(lr){3-3}\cmidrule(lr){4-4}\cmidrule(lr){5-5}\cmidrule(lr){6-6}
    &Accuracy $\uparrow$&Accuracy $\uparrow$&Accuracy $\uparrow$&Accuracy $\uparrow$&Accuracy $\uparrow$\\
    \hline
    GCN&	92.92±0.12&	92.70±0.20&	89.65±0.52&	61.76±0.14&71.74±0.29\\
    GAT&	93.61±0.14&	93.87±0.11&	90.78±0.13&	64.47±0.18&71.82±0.23\\
    APPNP&	94.49±0.07&	94.32±0.14&	90.18±0.17&	65.16±0.28&71.90±0.25\\
    \hline
    Scattering&94.77±0.33&	92.10±0.61&	85.68±0.71&	57.65±0.84&	66.23±0.19\\
    Scattering GCN&95.18±0.30&	93.07±0.42&	88.83±0.44	&61.14±1.13&	71.18±0.76\\
    \hline
    SGWT& 94.81±0.23&	92.45±0.62&	85.19±0.59&	55.04±1.12&	69.08±0.30\\
    GWNN& 90.75±0.59&	\underline{94.45±0.45}&	90.75±0.59&	64.19±0.79&	71.13±0.47\\
    UFGConvS&95.33±0.27&	93.98±0.59&	88.68±0.39&	61.25±0.93&	70.04±0.22\\
    UFGConvR&\underline{95.46±0.33}&	94.34±0.34&	89.29±0.46&	62.43±0.80&	71.97±0.12\\
    WaveShrink-ChebNet& 94.90±0.30&	93.54±0.90&	88.20±0.65&	58.98±0.69&	OOM\\
    DEFT&95.04±0.32&	94.35±0.44&	91.63±0.52&	\underline{68.01±0.86}&	72.01±0.20\\
    WaveNet&94.91±0.29&	94.09±0.63&	\underline{92.06±0.33}	&57.65±1.05	&71.37±0.14\\
    \hline
    SEA-GWNN&95.11±0.37&	94.35±0.50&	89.88±0.64	&66.74±0.79	&\underline{72.64±0.21}\\
    \hline
    \textbf{\M} (ours)&\textbf{95.89±0.34}&	\textbf{95.37±0.44}&	\textbf{92.26±0.18}&	\textbf{69.14±0.78}&\textbf{73.01±0.18}\\
  \toprule[1.2pt]
  \end{tabular}}
  \vskip -0.1in
\end{table*}
\section{Why do we need decomposition?}
\label{why do}
As shown in Fig.~\ref{model} (b), odd and even terms of Chebyshev polynomials meet the requirements on constructing wavelet after decomposition and transform. Additionally, each term is apt to be obtained according to the iteration formula, while infinite number of terms guarantee the expressiveness of the final composed wavelet. Next, we compare our decomposition solution with other related techniques:

$\bullet$ \textit{Constructing wavelet via Chebyshev polynomials.} Previous wavelet-based GNNs leverage Chebyshev polynomials with two purposes. (1) Approximate wavelets of pre-defined forms. SGWT~\citep{gw}, GWNN~\citep{gwnn} and UFGConvS/R~\cite{ufg} follow this line. They firstly fix the shape of wavelets as cubic spline, exponential or high-pass/low-pass filters, followed by the approximation via Chebyshev polynomials. In this pipeline, wavelet fails to learn further and suit the dataset and task at hand. (2) Compose a new wavelet. DEFT~\citep{deft} employs an MLP or GNN network to freely learn the coefficients before each Chebyshev basis. These coefficients are optimized according to the training loss, but loose the constraint on wavelet admissibility criteria.

$\bullet$ \textit{No decomposition.} If we uniformly learn the coefficients for all Chebyshev terms without decomposition, \M\ degrades to a variant similar to ChebNet~\citep{chebnet}. However, mixture rather than decomposition blends the signals from different ranges, and the final spatial ranges cannot be precisely predicted and controlled. 

We provide numerical comparison and spectral visualization in section~\ref{The effectiveness of general wavelet bases} for \M\ against these related studies.
\section{Numerical Experiments}
\label{Numerical Experiments}
In this section, we evaluate the performance of \M\ on both short-range and long-range benchmarks using the following datasets: (1) \textit{Datasets for short-range tasks:} \texttt{CS}, \texttt{Photo}, \texttt{Computer} and \texttt{CoraFull} from the PyTorch Geometric (PyG)~\citep{pyg}, and one large-size graph, i.e. \texttt{ogbn-arxiv} from Open Graph Benchmark (OGB)~\citep{ogb} (2) \textit{Datasets for long-range tasks:} \texttt{PascalVOC-SP (VOC)}, \texttt{PCQM-Contact (PCQM)}, \texttt{COCO-SP (COCO)}, \texttt{Peptides-func (Pf)} and \texttt{Peptides-struct (Ps)} from LRGB~\citep{lrgb}. Please refer to Appendix~\ref{Implementation details} for implementation details and Appendix~\ref{Datasets and baselines} for details of datasets. 


\subsection{Benchmarking \M}
\label{saeaseasesaeasease}
\vspace{-7pt}
\begin{table*}[ht]
\scriptsize
  \caption{Qualified results on long-range tasks compared to baselines. \textbf{Bold}: Best, \underline{Underline}: Runner-up, OOM: Out-of-memory, All results are reproduced based on source codes.}
  \label{graph_com}
  \vskip 0.05in
  \centering
  \setlength{\tabcolsep}{6mm}{
  \renewcommand\arraystretch{1.1}
  \begin{tabular}{cccccc}
    \toprule[1.2pt]
    \multirow{2}*{Model} & VOC  & PCQM & COCO  & Pf & Ps \\
    \cmidrule(lr){2-2}\cmidrule(lr){3-3}\cmidrule(lr){4-4}\cmidrule(lr){5-5}\cmidrule(lr){6-6}
    &F1 score $\uparrow$&MRR $\uparrow$&F1 score $\uparrow$&AP $\uparrow$&MAE $\downarrow$\\
    \hline
    GCN&12.68±0.60&32.34±0.06&08.41±0.10&59.30±0.23&34.96±0.13\\
    GINE&12.65±0.76&31.80±0.27&13.39±0.44&54.98±0.79&35.47±0.45\\
    GatedGCN&28.73±2.19&32.18±0.11&26.41±0.45&58.64±0.77&34.20±0.13\\
    \hline
    Scattering&16.58±0.49&	33.90±0.27&	16.44±0.79&	56.80±0.38&	26.77±0.11\\
    Scattering GCN&30.45±0.36&	33.73±0.45&	30.27±0.60&	62.87±0.64&	26.43±0.20\\
    \hline
    SGWT&31.22±0.56&	34.04±0.05&	\underline{32.97±0.53}&	60.23±0.27&	25.39±0.21\\
    GWNN&25.60±0.56&	32.72±0.08&	13.39±0.44&	65.47±0.48&	27.34±0.04\\
    UFGConvS&31.27±0.39&	33.94±0.24&	23.15±0.55&	65.83±0.75&	27.08±0.58\\
    UFGConvR&31.08±0.33&	34.08±0.20&	26.02±0.48&	65.29±0.82&	27.50±0.21\\
    WaveShrink-ChebNet&18.80±0.85&	32.56±0.11&	11.12±0.46&	61.12±0.53&	27.45±0.06\\
    DEFT&\underline{35.98±0.20}&	\underline{34.25±0.06}&	30.14±0.49&	66.95±0.63&	\underline{25.06±0.13}\\
    WaveNet&28.60±0.15&	33.19±0.20&	23.06±0.18&	64.63±0.27&	25.88±0.01\\
    \hline
    SEA-GWNN&31.97±0.55&	29.89±0.26&	24.33±0.23&	\underline{68.75±0.20}&	25.64±0.31\\
    \hline
    \textbf{\M} (ours)&\textbf{41.63±0.19}&\textbf{34.50±0.02}&\textbf{35.96±0.22}&\textbf{69.73±0.43}&\textbf{24.83±0.11}\\
  \toprule[1.2pt]
  \end{tabular}}
  \vskip -0.1in
\end{table*}

For short-range (S) datasets, we follow the settings from~\citep{nagphormer}.
For \texttt{ogbn-arxiv}, we use the public splits in OGB~\citep{ogb}. 
For long-range datasets, we adhere to the experimental configurations outlined in~\citep{lrgb}. The selected baselines belong to four categories, i.e., classical GNNs \{GCN~\cite{gcn}, GAT~\cite{gat}, APPNP~\cite{appnp}, GINE~\cite{gine1} and GatedGCN~\cite{gated}\}, graph scattering network \{Scattering~\cite{scattering} and Scattering GCN~\cite{Scattering_gcn}\}, spectral graph wavelet network \{SGWT~\cite{gw}, GWNN~\cite{gwnn}, UFGConvS~\cite{ufg}, UFGConvR~\cite{ufg}, WaveShrink~\cite{shrink}, DEFT~\cite{deft} and WaveNet~\cite{yang2024wavenet}\} and wavelet lifting transform \{SEA-GWNN~\cite{deb2024sea}\}~\footnote{For SGWT and Scattering, we concatenate the filtered signals from all bases and apply an MLP to get the final embeddings.}. The results of the comparison with SOTA models are shown in Table~\ref{node_com} and~\ref{graph_com}, where our \M\ demonstrates the best results on all datasets. Remarkably, the improvement on VOC achieves up to 11.83\%, implying the superior long-range information perception.

In the experiments conducted on the five short-range datasets, the model is required to prioritize local information, while the five long-range datasets necessitate the handling of distant interactions. The results clearly demonstrate that the proposed \M\ consistently outperforms traditional graph convolutions and graph wavelets in effectively aggregating both local and long-range information.

\subsection{Effectiveness of matrix-valued kernel}
\label{matrix-value}
The proposed matrix-valued kernel and weight-sharing strategy mark an advancement over conventional graph convolution, particularly in the context of processing wavelet-based signals. In this section, we conduct a comprehensive analysis of the effectiveness of these two architectural innovations.

As shown in Table~\ref{vectorvalue}, the matrix-valued kernel consistently outperforms its vector-valued counterpart. This improvement suggests that increasing the expressiveness of the kernel—through a higher parameter capacity—enhances the model’s ability on feature learning.
\begin{table}[h]
    \centering
    \small
\vskip -0.15in
    \caption{Compare \textit{Matrix-valued} and \textit{Vector-valued} kernels.}
    \vskip 0.05in
    \begin{tabular}{ccc}
         \toprule[1pt]
    Kernel & Computer (Accuracy $\uparrow$) & Ps (MAE $\downarrow$) \\
    \hline
    \textit{Vector-valued} & 89.96 & 25.30\\
    \textbf{\textit{Matrix-valued}} & \textbf{92.26} &\textbf{24.83}\\
    \toprule[1pt]
    \end{tabular}
    \label{vectorvalue}
\vskip -0.15in
\end{table}

In addition, Table~\ref{nonsharing} examines the impact of weight sharing across spectral frequencies. Assigning distinct kernels to individual frequencies does not improve performance and, even results in degradation. This decline is likely due to overfitting caused by the large number of parameters introduced in the non-sharing setup. Specifically, non-sharing kernels require a mapping from each eigenvalue embedding to a unique transformation matrix, defined as $f: \mathbb{R}^d\rightarrow\mathbb{R}^{d\times d}$, which is implemented using a multi-layer perceptron (MLP) with a weight dimension of $\mathbb{R}^{d\times d\times d}$, $d$ is the embedding dimension. For instance, when d=96 in \texttt{Ps}, this results in an approximate increase about 876K parameters.
\begin{table}[h]
    \footnotesize
\vskip -0.15in
    \caption{Compare $\textit{sharing}$ and $\textit{non-sharing}$ kernel weights.}
    \vskip 0.05in
    \centering
    \setlength{\tabcolsep}{1mm}{
    \renewcommand\arraystretch{1.1}
    \begin{tabular}{ccc}
         \toprule[1pt]
    Result (Parameters) & Computer (Accuracy $\uparrow$) & Ps (MAE $\downarrow$) \\
    \hline
    \textit{Non-sharing} & 90.51 (535k) & 26.22 (1,410k))\\
    \textbf{\textit{Sharing}} & \textbf{92.26 (167k)} & \textbf{24.83 (534k)}\\
    \toprule[1pt]
    \end{tabular}}
    \label{nonsharing}
\vskip -0.25in
\end{table}
 
\subsection{Effectiveness of learnable wavelet bases}
\label{The effectiveness of general wavelet bases}
In this section, we compare the learnt wavelet bases from \M\ with other baselines, including five graph wavelets (i.e. SGWT~\citep{gw}, UFGConvS/R~\cite{ufg}, DEFT~\citep{deft}, GWNN~\citep{gwnn} and WaveNet~\cite{yang2024wavenet}). We additionally evaluate ChebNet*, a variant of our \M\, where the only change is to combine odd and even terms without decomposition. Therefore, the improvement of \M\ over ChebNet* reflects the effectiveness of decoupling operation. The numerical comparison on \texttt{Computer} and \texttt{PascalVOC-SP} has been shown in Table.~\ref{node_com} and~\ref{graph_com}, which demonstrates obvious gains from \M\, especially on long-range \texttt{PascalVOC-SP}. The ChebNet* gets 89.85 and 36.45 separately on \texttt{Computer} and \texttt{PascalVOC-SP}, still inferior to \M.~\footnote{We explore more differences between ChebNet~\cite{chebnet} and our \M\ in Appendix~\ref{more comparisons}.}

\begin{figure*}[t]
  \centering
  \includegraphics[scale=0.16]{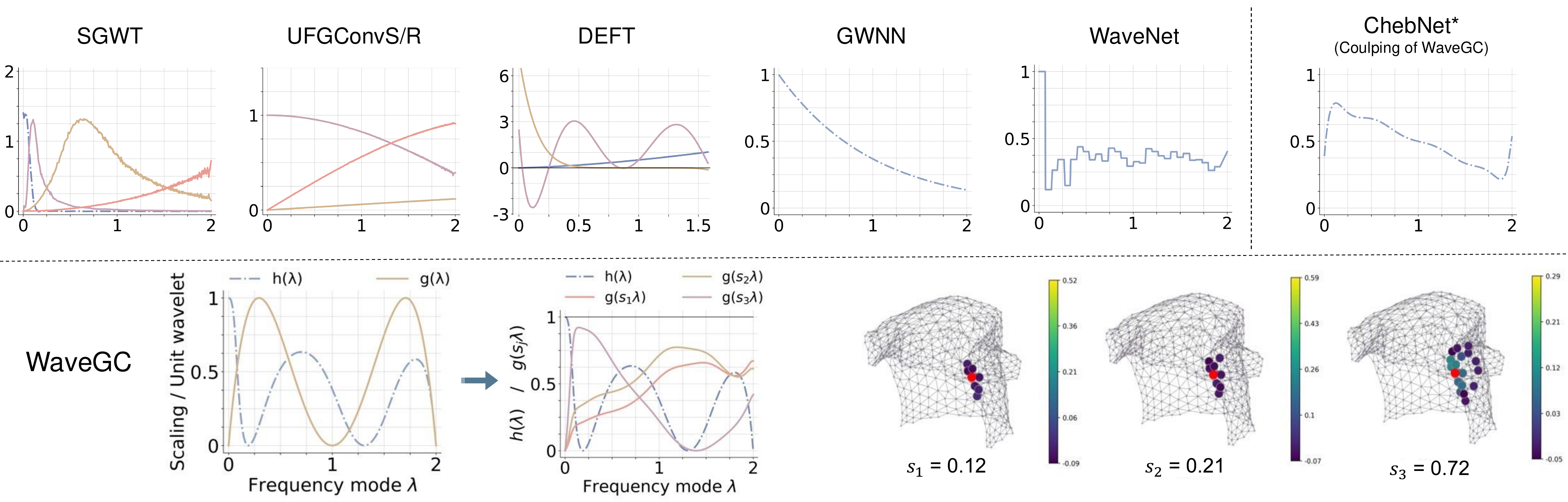}
    \vspace{-0.25cm}
  \caption{The spectral and spatial visualization of different bases on \texttt{PascalVOC-SP}.}
  \label{visul}
  \vskip -0.15in
\end{figure*}

To address the performance gap observed on the VOC dataset, we provide insights through the spectral visualization of various bases in Fig.~\ref{visul}. 
Upon examination of various wavelets, those from SGWT and UFGConvS/R meet admissibility criteria with multiple resolutions, but these lines are not adaptive. DEFT outputs multiple bases with unpredictable shapes, so it is hard to strictly restrain these outputs as wavelets. GWNN adopts one exponential wavelet base, omitting information from different ranges as well as not meeting criteria. WaveNet and ChebNet* blend local and distant information in spatial space, hampering the decision on the best range. For our \M, Fig.~\ref{visul} intuitively demonstrates that the unit wavelet got by our decoupling of Chebyshev polynomials strictly meets the admissibility criteria, as Eq.~\eqref{condition}, while the corresponding base scaling function supplements the direct current signals at $\lambda=0$. After integration of learnable scales, the final wavelets also meet criteria and adapt to the demand on multiresolution. The plot of $G(\lambda)=h(\lambda)^2+\sum_{j=1}^3g(s_j\lambda)^2$ as a black dashed line (located at 1) confirms the construction of tight frames via normalization technique. Fig.\ref{visul} also depicts the signal distribution over the topology centered on the target node (the red-filled circle). As the scale $s_j$ increases, the receptive field of the central node expands. Once again, this visualization intuitively confirms the capability of \M\ to aggregate both short- and long-range information simultaneously but distinguishingly. More analyses are given in Appendix~\ref{visualization on}.

\subsection{Ablation study}
\begin{table}[h]\vskip -0.15in
  \caption{Results of the ablation study. \textbf{Bold}: Best.}
  \label{abla}
  \small
  \vskip 0.05in
  \centering
  \setlength{\tabcolsep}{5mm}{
  \renewcommand\arraystretch{1.1}
  \begin{tabular}{ccc}
    \toprule[1.2pt]
    \multirow{2}*{Variants} & Computer & Ps\\
    \cmidrule(lr){2-2}\cmidrule(lr){3-3}&Accuracy $\uparrow$&MAE $\downarrow$\\
    \hline
    \M&  \textbf{92.26}&\textbf{24.83}\\
    \hline
    w/o wavelet& 89.65& 34.20\\
    \hline
    w/o MPNN&	90.89& 25.04\\
    w/o h($\lambda$)&	90.57&	25.12\\
    w/o g($s\lambda$)&	90.87&	25.09\\
    \toprule[1.2pt]
  \end{tabular}}\vskip -0.05in
\end{table}

In this section, we conduct an ablation study of our \M\ to assess the effectiveness of each component, and the corresponding results are presented in Table~\ref{abla}. The evaluation is conducted on \texttt{Peptides-struct} (long-range) and \texttt{Computer} (short-range).

Given the hybrid network (Fig.~\ref{mpgnn_wave}), we firstly remove the MPNN part (i.e., `w/o MPNN') and wavelet part (i.e., `w/o wavelet'), respectively. Both ablations degrade model performances, where `w/o wavelet' decline more. To avoid interference from MPNN part, we base on `w/o MPNN', and continue to exclude scaling term (i.e., `w/o $h(\lambda)$'), wavelet terms (i.e., `w/o $g(s\lambda)$') and tight frame constrains (i.e., `w/o tight frame'). Then, both the scaling function basis $h(\lambda)$ and wavelet bases $g(s\lambda)$ are essential components of our \M. In particular, neglecting $h(\lambda)$ results in a larger drop in performance on both short-range and long-range cases, emphasizing the crucial role of low-frequency information.

\subsection{Complexity analysis}
The main complexity of WaveGC is the eigen-decomposition process, involving $O(N^3)$. This is practical for the small-to-medium graphs used in all long-range and some short-range benchmarks, where detailed spectral modeling is critical. To accelerate the decomposition on large-scale graph (e.g., ogbn-arxiv), we may adopt randomized SVD~\cite{halko2009finding} with complexity $O(N^2\log K)$, where we only pick the top $K$ eigenvectors.

\begin{table}[h]
\vskip -0.15in
  \caption{Training and EVD time on short- and long-range datasets.}
  \label{t_s}
  \small
  \vskip 0.05in
  \centering
  \setlength{\tabcolsep}{0.9mm}{
  \renewcommand\arraystretch{1.2}
  \begin{tabular}{cccccc}
    \toprule[1.2pt]
    Short-range &CS& Photo& Computer& CoraFull& ogbn-arxiv\\
    \hline
    Training (min) & 5.70 &0.95 &4.87 &22.00 &36.67\\
    \hline
    EVD (min) & 2.82 &0.32 & 1.44& 3.49 & 21.69\\
    \hline
    \hline
    Long-range &VOC &PCQM &COCO &Pf &Ps\\
    \hline
    Training (h) & 4.02 &12.33 &45.40 &1.88 &1.32\\
    \hline
    EVD (h)& 0.05 &0.21 & 0.58& 0.02 & 0.02\\
    \toprule[1.2pt]
  \end{tabular}}\vskip -0.05in
\end{table}

Fig.~\ref{t_s} presents a direct comparison of the training time and EVD time across both short-range and long-range datasets. As shown, the time required for EVD is consistently lower than that of training across all datasets, with the difference being particularly significant in the long-range cases. Furthermore, the EVD operation is performed only once before training, and it is a prerequisite for most graph wavelet baselines. To further reduce complexity, we propose a fully polynomial-based approximation that removes the need for EVD, achieving total complexity of $O(N)$. More details are given in Appendix~\ref{ssssssss}.

\paragraph{Other experiments}
In Appendix~\ref{Time and space complexity analysis}, we analyze the complexity and report the running time for \M\ and other spectral graph wavelets. Our model shows shorter running times than competitive spectral models while being significantly more accurate. In Appendix~\ref{hperpara}, we test the sensitivity of two important hyper-parameters.
\section{Conclusion}
\label{conclusion}
In this study, we proposed a novel graph convolution operation based on wavelets (\M), establishing its theoretical capability to capture information at both short and long ranges through a multi-resolution approach.



\section*{Acknowledgements}
XB is supported by NUS Grant R-252-000-B97-133 and MOE AcRF T1 Grant 251RES2423. MB is partially supported by the EPSRC Turing AI World-Leading Research Fellowship No. EP/X040062/1 and EPSRC AI Hub No. EP/Y028872/1. The authors would like to express their gratitude to the reviewers for their feedback, which has improved the clarity and contribution of the paper.

\section*{Impact Statement}
This paper presents work whose goal is to advance the field of 
Machine Learning. There are many potential societal consequences 
of our work, none which we feel must be specifically highlighted here.

\bibliography{bib}
\bibliographystyle{icml2025}

\newpage
\appendix

\appendix
\onecolumn
\section{Theoretical Proof}
\label{all proof}
Firstly, we give two auxiliary but indispensable lemma and theorem. Let starts from the formula $\sigma(\Psi_{s}HW)$. In this equation, we bound the first derivate of non-linear function as $|\sigma'|<c_\sigma$, and set $||W||\leq w$, where $||\cdot||$ is the operator norm. First, we give an upper bound for each entry in $\Psi_{s}$.
\begin{lemma}[\textbf{Upper bound for graph wavelet}]
\label{aa}
    Let $\Psi=\bm{U}g(\Lambda)\bm{U}^T$. Given a large even number $K>0$, then for $\forall i,j\in V\times V$, we have:
    \begin{equation}
        (\Psi_s)_{ij}<\left(\alpha(\hat{\bm{A}})^{K/2}s^K\right)_{ij},\quad \alpha=\frac{C\cdot2^K(K+1)}{K!}.
    \end{equation}
\end{lemma}
The proof is given in Appendix~\ref{proof_aa}. In this lemma, we assume $g$ is smooth enough at $\lambda=0$. For fair comparison with traditional K-hop message passing framework $\sigma(\sum_{j=0}^K\tau_jA^jHW)$, we just test the flexibility with the similar form $\sigma(\Psi_{s}HW)$. In this case, we derive the depth $m_\Psi$ necessary for this wavelet basis $\Psi_{s}$ to induce the amount of mixing $\text{mix}_{y_G}(a, b)$ between two nodes $a$ and $b$.
\begin{theorem}[\textbf{The least depth for mixing}]
\label{bb}
    Given commute time $\tau(a,b)$~\citep{commute} and number of edges $|E|$. If $\Psi_{s}$ generates mixing $\text{mix}_{y_G}(b, a)$, then the number of layers $m_\Psi$ satisfies
    \begin{equation}
    \begin{aligned}
    \label{asfjisfj}
        m_\Psi \geq \frac{\tau(a,b)}{2K}+\frac{2|E|}{K\sqrt{d_ad_b}}\left[\frac{\text{mix}_{y_G}(b, a)}{\gamma(\alpha^2s^{2K})^{m_\Psi}}-
        \frac{1}{\lambda_1}(\gamma+|1-\lambda^*|^{Km_\Psi+1})\right],
    \end{aligned}
    \end{equation}
    where $d_a$ and $d_b$ are degrees of two nodes, $\gamma=\sqrt{\frac{d_{max}}{d_{min}}}$, and $|1-\lambda^*|=\max_{0<n\leq N-1}|1-\lambda_n|<1$.
\end{theorem}
The proof is given in Appendix~\ref{proof_bb}. In the following subsections, we firstly prove these lemma and theorem, and finally give the complete proof of Theorem~\ref{cc}.

\subsection{Proof of Lemma~\ref{aa} (Upper bound for graph wavelet)}
\label{proof_aa}
\begin{proof}
We aim to investigate the properties of filters $\Psi_{s_j}=\bm{U}g(s_j\lambda)\bm{U}^\top$ to capture both global and local information, corresponding to the cases $s_j\rightarrow0$ and $s_j\rightarrow\infty$, respectively.
In the former case, as $s_j$ approaches zero, $g(s_j\lambda)$ tends towards $g(0)$. For the latter case, the spectral information becomes densely distributed and concentrated near zero. Hence, the meaningful analysis of $g(\lambda)$ primarily revolves around $\lambda=0$.
Expanding $g(\lambda)$ using Taylor's series around $\lambda=0$, we get:
\begin{equation}
\label{aa_1}
g(\lambda)=\sum\limits_{k=0}^KC_k\frac{\lambda^k}{k!}+g^{(K+1)}(\lambda^*)\frac{\lambda^{K+1}}{(K+1)!}\approx \sum\limits_{k=0}^KC_k\frac{\lambda^k}{k!},
\end{equation}
where we neglect the high-order remainder term. Next, we have


\begin{subequations}
    \begin{align}                       (\Psi)_{ij}&=\left(\bm{U}g(\Lambda)\bm{U}^T\right)_{ij}=\left(\sum\limits_{k=0}^KC_k\frac{\hat{\mathcal{\bm{L}}}^k}{k!}\right)_{ij}\notag\\
        &=\left(\sum\limits_{k=0}^K\frac{C_k}{k!}(\bm{I}-\hat{\bm{A}})^k\right)_{ij}=\left(\sum\limits_{k=0}^K\frac{C_k}{k!}\sum_{p=0}^k\tbinom{k}{p}(-\hat{\bm{A}})^p\right)_{ij}\notag\\
        &<\left(\sum\limits_{k=0}^K\frac{C_k}{k!}\sum_{p=0}^k\tbinom{k}{p}(\hat{\bm{A}})^p\right)_{ij}=\left(\sum\limits_{k=0}^K\frac{C_k}{k!}\sum_{p=0}^k\frac{k!}{(k-p)!p!}(\hat{\bm{A}})^p\right)_{ij}\notag\\
        &=\left(\sum\limits_{k=0}^KC_k\sum_{p=0}^k\frac{(\hat{\bm{A}})^p}{(k-p)!p!}\right)_{ij}\label{aa_3}.
    \end{align}
\end{subequations}

We introduce a new parameter $\mu=\frac{\left(\sum_{k=0}^{K-1}C_k\sum_{p=0}^k\frac{(\hat{\pmb{A}})^p}{(k-p)!p!}\right)_{ij}}{\left(C_{K}\sum_{p=0}^K\frac{(\hat{\pmb{A}})^p}{(K-p)!p!}\right)_{ij}}$, so the above relation becomes:
\begin{equation}
(\Psi)_{ij} < \left((\mu+1)C_K\sum_{p=0}^K\frac{(\hat{\pmb{A}})^p}{(K-p)!p!}\right)_{ij}=\left(C\sum_{p=0}^K\frac{(\hat{\bm{A}})^p}{(K-p)!p!}\right)_{ij},
\end{equation}
where we set $C=(\mu+1)C_K$. Then, let us explore the expression $\epsilon_{ij}^p=\frac{(\hat{\bm{A}})_{ij}^p}{(K-p)!p!}$. First, we will address the denominator $(K-p)!p!$. As $p$ increases, this denominator experiences a sharp decrease followed by a rapid increase. The minimum value occurs at $(K/2)!(K/2)!$ when $p=K/2$, assuming $K$ is even.
Second, let's analyze the numerator $(\hat{\bm{A}})_{ij}^p$, which involves repeated multiplication of $\hat{\bm{A}}$. According to Theorem 1 in~\citep{deeper}, this repeated multiplication causes $(\hat{\bm{A}})^p$ to converge to the eigenspaces spanned by the eigenvector $D^{-1/2}\bm{1}$ of $\lambda=0$, where $\bm{1}=(1, 1,\dots, 1)\in\mathbb{R}^n$~\footnote{Simple proof. $(\hat{\bm{A}})^p=\bm{U}(\bm{I}-\bm{\Lambda})^p\bm{U}^\top=\sum_{i=0}^n(1-\lambda_i)^p\bm{u_1}\bm{u_1}^\top$. Provided only $1-\lambda_0=1$ and $1-\lambda_i\in(-1, 1)$ for other eigenvalues, with $p\rightarrow\infty$, only $(1-\lambda_0)^p=1$ but $(1-\lambda_i)^p\rightarrow0$. Thus, we have $(\hat{\bm{A}})^p\rightarrow\bm{u_1}\bm{u_1}^\top$, where $\bm{u_1}=D^{-1/2}\bm{1}$}.
Then, let us assume there exists a value $p^*$ beyond which the change in $(\hat{\bm{A}})^p$ becomes negligible. Given that $K$ is a large even number, we can infer that $K/2\gg p^*$. Thus, when $(K-p)!p!$ sharply decreases, $(\hat{\bm{A}})^p$ has already approached a stationary state. Consequently, $\max\epsilon_{ij}^p=\frac{(\hat{\bm{A}})_{ij}^{K/2}}{(K/2)!(K/2)!}$, where the denominator reaches its minimum. Thus, we have
\begin{subequations}
\begin{align}
    (\Psi)_{ij}&<\left(C\sum_{p=0}^K\frac{(\hat{\bm{A}})^p}{(K-p)!p!}\right)_{ij}\notag\\
    &<C(K+1)\left(\frac{(\hat{\bm{A}})^{K/2}}{(K/2)!(K/2)!}\right)_{ij}\notag\\
    &<\left(\frac{C\cdot2^K(K+1)}{K!}(\hat{\bm{A}})^{K/2}\right)_{ij}\label{aa_4}.
\end{align}
\end{subequations}

We have $\frac{1}{(K/2)!(K/2)!} < \frac{2^K}{K!}$ given that
\begin{equation}
    \begin{aligned}
        (K/2)!(K/2)!&=(\frac{K}{2}\cdot\frac{K-2}{2}\dots\frac{4}{2}\cdot\frac{2}{2})(\frac{K}{2}\cdot\frac{K-2}{2}\dots\frac{4}{2}\cdot\frac{2}{2})\\
        &>(\frac{K}{2}\cdot\frac{K-2}{2}\dots\frac{4}{2}\cdot\frac{2}{2})(\frac{K-1}{2}\cdot\frac{K-3}{2}\dots\frac{3}{2}\cdot\frac{1}{2})\\
        &=\underbrace{\frac{K\cdot K-1\cdot K-2 \cdot K-3\dots4\cdot3\cdot2\cdot1}{2\cdot2\cdot2\cdot2\dots2\cdot2\cdot2\cdot2}}_{\text{\small K terms}}=\frac{K!}{2^K}.
    \end{aligned}
\end{equation}
With $\alpha=\frac{C\cdot2^K(K+1)}{K!}$ and scale $s$, Eq.~(\ref{aa_4}) can be finally written as
\begin{equation}
    (\Psi_{s})_{ij}<\left(\alpha(\hat{\bm{A}})^{K/2}s^K\right)_{ij}.
\end{equation}


\end{proof}

\subsection{Proof of Theorem~\ref{bb} (The least depth for mixing)}
\label{proof_bb}
For this section, we mainly refer to the proof from~\citep{how}.

\textbf{Preliminary.} For simplicity, we follow~\citep{how} to denote some operations utilized in this section. As stated, we consider the message passing formula $\sigma(\Psi_{s}HW)$. First, we denote $\bm{h_a}^{(l), \alpha}$ as the $\alpha$-th entry of the embedding $\bm{h}_a^{(l)}$ for node $a$ at the $l$-th layer. Then, we rewrite the formula as:
\begin{equation}
\label{simple_}
    \bm{h}^{(l), \alpha}_a=\sigma(\widetilde{\bm{h}}_a^{(l-1), \alpha}),\quad 1\leq\alpha\leq d,
\end{equation}
where $\widetilde{\bm{h}}_a^{(l-1), \alpha}=(\Psi_sHW)_a$ is the entry $\alpha$ of the pre-activated embedding of node $a$ at layer $l$. Given nodes $a$ and $b$, we denote the following differentiation operations:
\begin{equation}
    \nabla_a\bm{h}_b^{(l)}:=\frac{\partial\bm{h}_b^{(l)}}{\partial\bm{x}_a},\quad \nabla_{ab}^2\bm{h}_i^{(l)}:=\frac{\partial^2\bm{h}_i^{(l)}}{\partial\bm{x}_a\partial\bm{x}_b}.
\end{equation}
Next, we firstly derive upper bounds on $\nabla_a\bm{h}_b^{(l)}$, and then on $\nabla_{ab}^2\bm{h}_i^{(l)}$.
\begin{lemma}
    \label{c.1}
    Given the message passing formula $\sigma(\Psi_{s}HW)$, let assume $|\sigma'|\leq c_\sigma$ and $||W||\leq w$, where $||\cdot||$ is the operator norm. For two nodes $a$ and $b$ after $l$ layers of message passing, the following holds:
    \begin{equation}
        ||\nabla_a\bm{h}_b^{(l)}||\leq(c_\sigma w)^l(\bm{B}^l)_{ba},
    \end{equation}
    where $\bm{B}_{ba}=\left(\alpha(\hat{\bm{A}})^{K/2}s^K\right)_{ba}$.
\end{lemma}
\begin{proof}
    If $l=1$ and we fix entries $1\leq\alpha, \beta\leq d$, then we have:
    \begin{equation}
        (\nabla_a\bm{h}_b^{(1)})_{\alpha\beta}=(\textrm{diag}(\sigma'(\widetilde{\bm{h}}_b^{(0)}))(\bm{W}^{(1)}\Psi_{ba}\bm{I}))_{\alpha\beta}.
    \end{equation}
    With Cauchy–Schwarz inequality, we bound the left hand side by
    \begin{subequations}        
    \begin{align}
        ||\nabla_a\bm{h}_b^{(1)}||&\leq||\textrm{diag}(\sigma'(\widetilde{\bm{h}}_b^{(0)}))||\cdot||\bm{W}^{(1)}\Psi_{ba}||\notag\\
        &\leq c_\sigma w\bm{B}_{ba}.\notag
    \end{align}
    \end{subequations}
    Next, we turn to a general case where $l>1$:
    \begin{equation}
            (\nabla_a\bm{h}_b^{(l)})_{\alpha\beta}=(\textrm{diag}(\sigma'(\widetilde{\bm{h}}_b^{(l-1)})(W\sum\limits_j\Psi_{bj}\nabla_a\bm{h}_j^{(m-1)}))_{\alpha\beta}.
    \end{equation}
    Then, we can use the induction step to bound the above equation:
    \begin{equation}
    \label{assafvscc}
        \begin{aligned}
            ||\nabla_a\bm{h}_b^{(l)}||&\leq(c_\sigma w)^{l}|\sum\limits_{j_0}\sum\limits_{j_1}\dots\sum\limits_{j_{l-2}}\Psi_{bj_0}\Psi_{j_0j_1}\dots\Psi_{j_{l-3}j_{l-2}}\Psi_{j_{l-2}a}|\\
            &\leq(c_\sigma w)^{l}(\bm{B}^l)_{ba}.
        \end{aligned}
    \end{equation}
In Eq.~(\ref{assafvscc}), we implicitly use $|\Psi_{s}^l|_{ba}<\left(\alpha(\hat{\bm{A}})^{K/2}s^K\right)^l_{ba}=\bm{B}^l_{ba}$. Similar to proof given in Appendix~\ref{proof_aa}, we can give the following proof:
\begin{equation}
\label{agifoqwehgiqw}
    \begin{aligned}
        |\Psi_s^l|_{ba}&=\left|\bm{U}g(s\Lambda)^l\bm{U}^T\right|_{ba}=\left|s^{lK}C^l\frac{\hat{\mathcal{\bm{L}}}^{lK}}{K!^l}\right|_{ba}\\
        &=\left|s^{lK}\frac{C^l}{K!^l}(\bm{I}-\hat{\bm{A}})^{lK}\right|_{ba}=\left|s^{lK}\frac{C^l}{K!^l}\sum_{p=0}^{lK}\tbinom{lK}{p}(-\hat{\bm{A}})^p\right|_{ba}\\
        &<\left(s^{lK}\frac{C^l}{K!^l}\sum_{p=0}^{lK}\tbinom{lK}{p}(\hat{\bm{A}})^p\right)_{ba}=\left(s^{lK}\frac{C^l}{K!^l}\sum_{p=0}^{lK}\frac{(lK)!}{(lK-p)!p!}(\hat{\bm{A}})^p\right)_{ba}\\
        &=\left(s^{lK}\frac{C^l(lK)!}{K!^l}\sum_{p=0}^{lK}\frac{(\hat{\bm{A}})^p}{(lK-p)!p!}\right)_{ba}<\left(s^{lK}\frac{C^l(lK)!}{K!^l}(lK+1)\left(\frac{(\hat{\bm{A}})^{lK/2}}{(lK/2)!(lK/2)!}\right)\right)_{ba}\\
        &<\left(s^{lK}\frac{C^l(lK)!}{K!^l}(lK+1)\frac{2^{lK}}{(lK)!}(\hat{\bm{A}})^{lK/2}\right)_{ba}=\left(s^{lK}\frac{C^l\cdot2^{lK}(lK+1)}{K!^l}(\hat{\bm{A}})^{lK/2}\right)_{ba}\\
        &<\left(s^{lK}\frac{C^l\cdot2^{lK}(K+1)^l}{K!^l}(\hat{\bm{A}})^{lK/2}\right)_{ba}=\left(\alpha(\hat{\bm{A}})^{K/2}s^K\right)^l_{ba},
    \end{aligned}
\end{equation}
where in the last line, we utilize the relation $lK+1<(K+1)^l$.
\end{proof}

\begin{lemma}
    \label{c.2}
    Given the message passing formula $\sigma(\Psi_{s}HW)$, let assume $|\sigma'|, |\sigma''|\leq c_\sigma$ and $||W||\leq w$, where $||\cdot||$ is operator norm. For nodes $i$, $a$ and $b$ after $l$ layers of message passing, the following holds:
    \begin{equation}
        ||\nabla_{ab}^2\bm{h}_i^{(l)}||\leq\sum\limits_{k=0}^{l-1}\sum\limits_{j\in V}(c_\sigma w)^{2l-k-1}w(\bm{B}^{l-k})_{jb}(\bm{B}^k)_{ij}(\bm{B}^{l-k})_{ja},
    \end{equation}
    where $\bm{B}_{ba}=\left(\alpha(\hat{\bm{A}})^{K/2}s^K\right)_{ba}$.
\end{lemma}
\begin{proof}
    Considering $\nabla_{ab}^2\bm{h}_i^{(l)}\in\mathbb{R}^{d\times(d\times d)}$, we refer to~\citep{how} to use the following ordering for indexing the columns:
    \begin{equation}
        \frac{\partial^2\bm{h}_i^{(l),\alpha}}{\partial x_b^\beta\partial x_a^\gamma}:=(\nabla_{ab}^2\bm{h}_i^{(l)})_{\alpha, d(\beta-1)+\gamma}.
    \end{equation}
    Similar to the proof of Lemma~\ref{c.1}, we firstly focus on $m=1$:
    \begin{equation}
        (\nabla_{ab}^2\bm{h}_i^{(1)})_{\alpha, d(\beta-1)+\gamma}=(\textrm{diag}(\sigma''(\widetilde{\bm{h}}_i^{(0), \alpha}))(\bm{W}^{(1)}\Psi_{ib}\bm{I})_{\alpha\gamma}\times(\bm{W}^{(1)}\Psi_{ia}\bm{I})_{\alpha\beta}.
    \end{equation}
    We bound the left-hand side as:
    \begin{equation}
        ||\nabla_{ab}^2\bm{h}_i^{(1)}||\leq(c_\sigma w)(w|\bm{B}_{ib}||\bm{B}_{ia}|).
    \end{equation}
    Then, for $m>1$:
    \begin{equation}
    \label{qwrqwqc}
        \begin{aligned}
            &(\nabla_{ab}^2\bm{h}_i^{(l)})_{\alpha, d(\beta-1)+\gamma}\\
            &=\underbrace{\textrm{diag}(\sigma''(\widetilde{\bm{h}}_i^{(l-1), \alpha})(W\sum\limits_j\Psi_{ij}\nabla_a\bm{h}_j^{(l-1)})\times(W\sum\limits_j\Psi_{ij}\nabla_b\bm{h}_j^{(l-1)})}_{\bm{R}}\\
            &+\underbrace{\textrm{diag}(\sigma'(\widetilde{\bm{h}}_i^{(l-1), \alpha})(\bm{W}^{(m)}\sum\limits_j\Psi_{ij}\nabla^2_{ab}\bm{h}_j^{(l-1)})}_{\bm{Z}}.
        \end{aligned}
    \end{equation}
    We denote $||\nabla_j\bm{h}_i^{(l-1)}||$ as $(D\bm{h}^{(l-1)})_{ij}$, and $||\nabla_{ab}^2\bm{h}_i^{(l-1)}||$ as $({D^2\bm{h}^{(l-1)}}_{ba})_{i}$. To bound $\bm{R}$, we deduce as follows:
    \begin{subequations}
        \begin{align}
            ||\bm{R}||&\leq c_\sigma(w\sum\limits_j\bm{B}_{ij}||\nabla_a\bm{h}_j^{(l-1)}||)\times(w\sum\limits_j\bm{B}_{ij}||\nabla_b\bm{h}_j^{(l-1)}||)\notag\\
            &=c_\sigma w(w\bm{B}D\bm{h}^{(l-1)})_{ib}(\bm{B}D\bm{h}^{(l-1)})_{ia}\notag\\
            &\leq c_\sigma w(w\bm{B}(c_\sigma w)^{l-1}\bm{B}^{l-1})_{ib}(\bm{B}(c_\sigma w)^{l-1}\bm{B}^{l-1})_{ia}\label{d.1}\\
            &=(c_\sigma w)^{2l-1}(w(\bm{B}^{l})_{ib}(\bm{B}^l)_{ia})\notag,
        \end{align}
    \end{subequations}
    where we utilize the conclusion from Theorem~\ref{c.1} in~(\ref{d.1}). For term $\bm{Z}$, we have:
    \begin{subequations}
        \begin{align}
            ||\bm{Z}||&\leq c_\sigma w(\bm{B}D^2\bm{h}^{(l-1)})_i\notag\\
            &\leq c_\sigma w\sum\limits_s\bm{B}_{is}\sum\limits_{k=0}^{l-2}\sum\limits_{j\in V}(c_\sigma w)^{2l-2-k-1}w(\bm{B}^{l-1-k})_{jb}(\bm{B}^{k})_{sj}(\bm{B}^{l-1-k})_{ja}\label{d.2}\\
            &=\sum\limits_{k=0}^{l-2}\sum\limits_{j\in V}(c_\sigma w)^{2l-2-k}(\bm{B}^{l-1-k})_{jb}(\bm{B}^{k+1})_{ij}(\bm{B}^{l-1-k})_{ja}\notag\\
            &=\sum\limits_{k=1}^{l-1}\sum\limits_{j\in V}(c_\sigma w)^{2l-1-k}(\bm{B}^{l-k})_{jb}(\bm{B}^{k})_{ij}(\bm{B}^{l-k})_{ja},\notag
        \end{align}
    \end{subequations}
    where in~(\ref{d.2}), we recursively use the Eq.~(\ref{qwrqwqc}). Finally, we finish the proof as:
    \begin{equation}
        \begin{aligned}
            ||\nabla_{ab}^2\bm{h}_i^{(l)}||&\leq||\bm{R}||+||\bm{Z}||\\
            &\leq \sum\limits_{k=0}^{l-1}\sum\limits_{j\in V}(c_\sigma w)^{2l-1-k}(\bm{B}^{l-k})_{jb}(\bm{B}^{k})_{ij}(\bm{B}^{l-k})_{ja}.
        \end{aligned}
    \end{equation}
\end{proof}

With Lemma~\ref{c.1} and~\ref{c.2}, now we give the following theorem.
\begin{theorem}
\label{qqplmvnr}
    Consider the message passing formula $\sigma(\Psi_{s}HW)$ with $m_\Psi$ layers, the induced mixing $\text{mix}_{y_G}(b, a)$ over the features of nodes $a$ and $b$ satisfies:
    \begin{equation}
        \textrm{mix}_{y_G}(b, a) \leq \sum\limits_{l=0}^{m_\Psi-1}(c_\sigma w)^{(2m_\Psi-l-1)}\left(w\left(\bm{B}^{m_\Psi-l}\right)^\top \textrm{diag}\left(\bm{1}^\top\bm{B}^l\right)\bm{B}^{m_\Psi-l}\right)_{ab},
    \end{equation}
    where $\bm{B}_{ba}=\left(\alpha(\hat{\bm{A}})^{K/2}s^K\right)_{ba}$ and $\bm{1}\in\mathbb{R}^{n}$ is the vector of ones.
\end{theorem}
\begin{proof}
    Here, we define the prediction function $y_G: N\times d\rightarrow d$ on $G$ as $y_G^{(m_\Psi)}=\texttt{Readout}(\bm{H}^{(m_\Psi)}\bm{\theta})$, where $\texttt{Readout}$ is to gather all nodes embeddings to get the final graph embedding, $\bm{H}^{(m_\Psi)}$ is the node embedding matrix after $m_\Psi$ layers and $\bm{\theta}$ is the learnable weight for graph-level task. If we set $\texttt{Readout}=\texttt{sum}$, we derive:
    \begin{subequations}
        \begin{align}
            \text{mix}_{y_G}(b, a)&=\max\limits_x\max\limits_{1\leq\beta, \gamma\leq d}\left|\frac{\partial^2{y_G}^{(m_\Psi)}(\bm{X})}{\partial\bm{x}_a^\beta\partial\bm{x}_b^\gamma}\right|\notag\\
            &\leq \sum_{i\in V}\left|\sum\limits_{\alpha=1}^d\theta_\alpha\frac{\partial^2h_i^{(m_\Psi), \alpha}}{\partial\bm{x}_a^\beta\partial\bm{x}_b^\gamma}\right|\notag\\
            &=\sum_{i\in V}||(\nabla_{ab}^2\bm{h}_i^{(m_\Psi)})^\top\bm{\theta}||\notag\\
            &\leq \sum_{i\in V}||\nabla_{ab}^2\bm{h}_i^{(m_\Psi)}||\label{e.1}\\
            &\leq\sum\limits_{k=0}^{m_\Psi-1}(c_\sigma w)^{(2m_\Psi-k-1)}\left(w\left(\bm{B}^{m_\Psi-k}\right)^\top diag\left(\bm{1}^\top\bm{B}^k\right)\bm{B}^{m_\Psi-k}\right)_{ab}\label{e.2},
        \end{align}
    \end{subequations}
\end{proof}
where in~(\ref{e.1}), we assume the norm $||\bm{\theta}||\leq 1$. In~(\ref{e.2}), we use the results from Lemma~\ref{c.2}. This upper bound still holds if $\texttt{Readout}$ is chosen as $\texttt{MEAN}$ or $\texttt{MAX}$~\citep{how}.

In theorem~\ref{qqplmvnr}, we can assume that $c_\sigma$ to be smaller or equal than one, which is satisfied by the majority of current active functions. Furthermore, considering the normalization (e.g., $L_2$ norm) on $W$, we assume $w<1$. With these two assumptions, the conclusion of theorem~\ref{qqplmvnr} is rewritten as:
\begin{equation}
\label{asfasghie}
    \text{mix}_{y_G}(b, a) \leq\sum\limits_{l=0}^{m_\Psi-1}\left(\left(\bm{B}^{m_\Psi-l}\right)^\top \textrm{diag}\left(\bm{1}^\top\bm{B}^l\right)\bm{B}^{m_\Psi-l}\right)_{ab}.
\end{equation}
With this new conclusion, we now turn to the proof of Theorem~\ref{bb}:
\begin{proof}
Firstly, $\textrm{diag}\left(\bm{1}^\top\bm{B}^l\right)_i=(\alpha s^K)^l(((\hat{\bm{A}})^{K/2})^l\bm{1})_i\leq\gamma(\alpha s^K)^l$ by using $(((\hat{\bm{A}})^{K/2})^l\bm{1})_i\leq\gamma$~\citep{how}. Then, we find
\begin{equation}
    \begin{aligned}
        \sum\limits_{l=0}^{m_\Psi-1}\left(\left(\bm{B}^{m_\Psi-l}\right)^\top \textrm{diag}\left(\bm{1}^\top\bm{B}^l\right)\bm{B}^{m_\Psi-l}\right)_{ab}&\leq\gamma\left(\sum\limits_{l=0}^{m_\Psi-1}\bm{B}^{2(m_\Psi-l)}\cdot(\alpha s^K)^l\right)_{ab}\\
        &<\gamma\left(\sum\limits_{l=0}^{m_\Psi-1}(\alpha(\hat{\bm{A}})^{K/2}s^K)^{2(m_\Psi-l)}\cdot(\alpha s^K)^l\right)_{ab}\\
        &<\gamma(\alpha s^K)^{2m_\Psi}\left(\sum\limits_{l=0}^{m_\Psi-1}\hat{\bm{A}}^{K(m_\Psi-l)}\right)_{ab}\\
        &=\gamma(\alpha s^K)^{2m_\Psi}\left(\sum\limits_{l=1}^{m_\Psi}\hat{\bm{A}}^{Kl}\right)_{ab}.
    \end{aligned}
\end{equation}
The following proof depends on \textit{commute time} $\tau(a,b)$~\citep{commute}, whose the definition is as follows using the spectral representation of the graph Laplacian~\citep{how}:
\begin{equation}
    \label{commute}
    \tau(a,b)=2|E|\sum_{n=0}^{N-1}\frac{1}{\lambda_n}\left(\frac{u_n(a)}{\sqrt{d_a}}-\frac{u_n(b)}{\sqrt{d_b}}\right)^2.
\end{equation}
Then, we have:
\begin{subequations}
\label{ajgvifdn}
    \begin{align}
        \left(\sum\limits_{l=1}^{m_\Psi}\hat{\bm{A}}^{Kl}\right)_{ab}&\leq\sum\limits_{l=0}^{Km_\Psi}\left(\hat{\bm{A}}^{l}\right)_{ab}\notag\\
        &=\sum\limits_{l=0}^{Km_\Psi}\sum_{n\geq0}(1-\lambda_n)^lu_n(a)u_n(b)\notag\\
        &=(Km_\Psi+1)\frac{\sqrt{d_ad_b}}{2|E|}+\sum_{n>0}\frac{1-(1-\lambda)^{Km_\Psi+1}}{\lambda_n}u_n(a)u_n(b)\label{bb_1}\\
        &=(Km_\Psi+1)\frac{\sqrt{d_ad_b}}{2|E|}+\sum_{n>0}\frac{1}{\lambda_n}u_n(a)u_n(b)-\sum_{n>0}\frac{(1-\lambda)^{Km_\Psi+1}}{\lambda_n}u_n(a)u_n(b)\notag.
    \end{align}
\end{subequations}
In Eq.~(\ref{bb_1}), we use $u_0(a)=\sqrt{\frac{d_a}{2|E|}}$. Then, from the definition of commute time, we can get:
\begin{equation}
\label{gnvjdf}
\begin{aligned}
    \sum_{n=1}^{N-1}\frac{1}{\lambda_n}u_n(a)u_n(b)&=\frac{-\tau(a,b)}{4|E|}\sqrt{d_ad_b}+\frac{1}{2}\sum_{n>0}\frac{1}{\lambda_n}(u_n^2(a)\sqrt{\frac{d_b}{d_a}}+u_n^2(b)\sqrt{\frac{d_a}{d_b}})\\
    &\leq\frac{-\tau(a,b)}{4|E|}\sqrt{d_ad_b}+\frac{1}{2{\lambda_1}}\left(\sqrt{\frac{d_a}{d_b}}+\sqrt{\frac{d_b}{d_a}}-\frac{\sqrt{d_ad_b}}{|E|}\right),
\end{aligned}
\end{equation}
where in the last inequation, we utilize the fact that $\sum_{n>0}u_n^2(a)=1-u_0^2(a)$ because $\{u_n\}$ is a set of orthonormal basis. Besides, we use $\lambda_n>\lambda_1, \forall n>1$. Next, we derive
\begin{equation}
\label{jsginbjf}
    \begin{aligned}
        -\sum_{n>0}\frac{(1-\lambda)^{Km_\Psi+1}}{\lambda_n}u_n(a)u_n(b)&\leq\sum_{n>0}\frac{|1-\lambda^*|^{Km_\Psi+1}}{\lambda_n}|u_n(a)u_n(b)||\\
        &\leq\frac{|1-\lambda^*|^{Km_\Psi+1}}{2\lambda_1}\sum_{n>0}(u_n^2(a)+u_n^2(b))\\
        &\leq\frac{|1-\lambda^*|^{Km_\Psi+1}}{2\lambda_1}\left(2-\frac{d_a+d_b}{2|E|}\right),
    \end{aligned}
\end{equation}
where $|1-\lambda^*|=\max_{0<n\leq N-1}|1-\lambda_n|<1$. Insert derivations~(\ref{gnvjdf}) and~(\ref{jsginbjf}) into~(\ref{ajgvifdn}), then
gather all above derivations:
\begin{equation}
    \begin{aligned}
        &\text{mix}_{y_G}(b, a)\leq\gamma(\alpha s^K)^{2m_\Psi}\left\{(Km_\Psi+1)\frac{\sqrt{d_ad_b}}{2|E|}-\frac{\tau(a,b)}{4|E|}\sqrt{d_ad_b}\right.\\
        &\left.+\frac{1}{2{\lambda_1}}\left(\sqrt{\frac{d_a}{d_b}}+\sqrt{\frac{d_b}{d_a}}-\frac{\sqrt{d_ad_b}}{|E|}\right)+\frac{|1-\lambda^*|^{Km_\Psi+1}}{2\lambda_1}\left(2-\frac{d_a+d_b}{2|E|}\right)\right\}\\
        &\leq\gamma(\alpha s^K)^{2m_\Psi}\sqrt{d_ad_b}\left(\frac{Km_\Psi}{2|E|}-\frac{\tau(a,b)}{4|E|}\right)+\frac{\gamma(\alpha s^K)^{2m_\Psi}}{2{\lambda_1}}\left(\sqrt{\frac{d_a}{d_b}}+\sqrt{\frac{d_b}{d_a}}\right)+\frac{\gamma(\alpha s^K)^{2m_\Psi}}{\lambda_1}|1-\lambda^*|^{Km_\Psi+1}.
    \end{aligned}
\end{equation}
In last inequation, we discard $\frac{\sqrt{d_ad_b}}{2|E|}\left[1-\frac{1}{\lambda_1}\left(1+\frac{|1-\lambda^*|^{Km_\Psi+1}}{2}\left(\sqrt{\frac{d_a}{d_b}}+\sqrt{\frac{d_b}{d_a}}\right)\right)\right]<0$ because $\lambda_1<1$. Then,
\begin{equation}
\label{gbrhdjskldf}
    \frac{\text{mix}_{y_G}(b, a)}{\gamma(\alpha s^K)^{2m_\Psi}\sqrt{d_ad_b}}\leq\frac{Km_\Psi}{2|E|}-\frac{\tau(a,b)}{4|E|}+\frac{1}{2{\lambda_1}\sqrt{d_ad_b}}\left(\sqrt{\frac{d_a}{d_b}}+\sqrt{\frac{d_b}{d_a}}+2|1-\lambda^*|^{Km_\Psi+1}\right).
\end{equation}
From~(\ref{gbrhdjskldf}), we can finally give the lower bound of $m_\Psi$ as:
\begin{equation}
    \begin{aligned}
        m_\Psi&\geq\frac{2|E|}{K}\left\{\frac{\tau(a,b)}{4|E|}+\frac{\text{mix}_{y_G}(b, a)}{\gamma(\alpha s^K)^{2m_\Psi}\sqrt{d_ad_b}}-\frac{1}{2{\lambda_1}\sqrt{d_ad_b}}\left(\sqrt{\frac{d_a}{d_b}}+\sqrt{\frac{d_b}{d_a}}+2|1-\lambda^*|^{Km_\Psi+1}\right)\right\}\\
        &>\frac{2|E|}{K}\left\{\frac{\tau(a,b)}{4|E|}+\frac{1}{\sqrt{d_ad_b}}\left[\frac{\text{mix}_{y_G}(b, a)}{\gamma(\alpha s^K)^{2m_\Psi}}-\frac{1}{2{\lambda_1}}\left(2\gamma+2|1-\lambda^*|^{Km_\Psi+1}\right)\right]\right\}\\
        &=\frac{2|E|}{K}\left\{\frac{\tau(a,b)}{4|E|}+\frac{1}{\sqrt{d_ad_b}}\left[\frac{\text{mix}_{y_G}(b, a)}{\gamma(\alpha^2 s^{2K})^{m_\Psi}}-\frac{1}{{\lambda_1}}\left(\gamma+|1-\lambda^*|^{Km_\Psi+1}\right)\right]\right\}\\
        &=\frac{\tau(a,b)}{2K}+\frac{2|E|}{K\sqrt{d_ad_b}}\left[\frac{\text{mix}_{y_G}(b, a)}{\gamma(\alpha^2 s^{2K})^{m_\Psi}}-\frac{1}{{\lambda_1}}\left(\gamma+|1-\lambda^*|^{Km_\Psi+1}\right)\right]
    \end{aligned}
\end{equation}
\end{proof}

\subsection{Proof of Theorem~\ref{cc} (Short-range and long-range receptive fields)}
\label{proof_cc}
\begin{proof}    
From theorem~\ref{bb}, we denote $L_{m_\Psi}=\frac{\tau(a,b)}{2K}+\frac{2|E|}{K\sqrt{d_ad_b}}\left[\frac{\text{mix}_{y_G}(b, a)}{\gamma(\alpha^2 s^{2K})^{m_\Psi}}-\frac{1}{{\lambda_1}}\left(\gamma+|1-\lambda^*|^{Km_\Psi+1}\right)\right]$. For K-order message passing $\sigma(\sum_{j=0}^K\tau_jA^jHW)$, $\tau_j \in[0, 1]$, we assume that $(\tau_PA^P)_{ba}$ is the maximum among $\{(\tau_0A^0)_{ba},\dots,(\tau_KA^K)_{ba}\}$. According to theorem~\ref{qqplmvnr}, we can get the similar conclusion, replacing $\bm{B}$ with $\bm{C}=(K+1)\tau_PA^P$. Then, we have the following proof:
\begin{proof}
Again, $\textrm{diag}\left(\bm{1}^\top\bm{C}^l\right)_i=((K+1)\tau_P)^l(A^{Pl})\bm{1})_i\leq\gamma((K+1)\tau_P)^l$. Then, we have
\begin{equation}
    \begin{aligned}
        \sum\limits_{l=0}^{m_A-1}\left(\left(\bm{C}^{m_A-l}\right)^\top \textrm{diag}\left(\bm{1}^\top\bm{C}^l\right)\bm{C}^{m_A-l}\right)_{ab}&\leq\gamma\left(\sum\limits_{l=0}^{m_A-1}\bm{C}^{2(m_A-l)}\cdot((K+1)\tau_P)^l\right)_{ab}\\
        &<\gamma\left(\sum\limits_{l=0}^{m_A-1}((K+1)\tau_PA^P)^{2(m_A-l)}\cdot((K+1)\tau_P)^l\right)_{ab}\\
        &<\gamma((K+1)\tau_P)^{2m_A}\left(\sum\limits_{l=0}^{m_A-1}\hat{\bm{A}}^{2P(m_A-l)}\right)_{ab}\\
        &=\gamma((K+1)\tau_P)^{2m_A}\left(\sum\limits_{l=1}^{m_A}\hat{\bm{A}}^{2Pl}\right)_{ab}\\
        &<\gamma(\sqrt{(K+1)\tau_P})^{4m_A}\left(\sum\limits_{l=1}^{2m_A}\hat{\bm{A}}^{Pl}\right)_{ab}.
    \end{aligned}
\end{equation}
\end{proof}
Following the rest proof of $L_{m_\Psi}$, replace $\{\alpha s^K, m_\Psi, K\}$ with $\{\sqrt{(K+1)\tau_P}, 2m_A, P\}$, and get the expression of $L_{m_A}$:
\begin{equation}
    L_{m_A}=\frac{\tau(a,b)}{2P}+\frac{2|E|}{P\sqrt{d_ad_b}}\left[\frac{\text{mix}_{y_G}(b, a)}{\gamma((K+1)^2 {\tau_P}^2)^{m_A}}-\frac{1}{{\lambda_1}}\left(\gamma+|1-\lambda^*|^{2Pm_A+1}\right)\right].
\end{equation}
Therefore, we have
\begin{equation}
\label{agvfds}
    L_{m_\Psi}\approx\frac{P}{K}L_{m_A}+\frac{2|E|}{K\sqrt{d_ad_b}}\left[\frac{\text{mix}_{y_G}(b, a)}{\gamma}\left(\frac{1}{(\alpha^2 s^{2K})^{m_\Psi}}-\frac{1}{((K+1)^2 {\tau_P}^2)^{m_A}}\right)\right],
\end{equation}
where we ignore $|1-\lambda^*|^{Km_\Psi+1}$ and $|1-\lambda^*|^{2Pm_A+1}$. Since $|1-\lambda^*|<1$ as shown in theorem~\ref{bb}, therefore $|1-\lambda^*|^{Km_\Psi+1}-|1-\lambda^*|^{2Pm_A+1}$ will be very small, especially when $m_\Psi$ and $m_A$ are large. From Eq.~(\ref{agvfds}), when $s\rightarrow\infty$, the relation becomes:
\begin{equation}
    L_{m_\Psi} \approx \frac{P}{K}L_{m_A}-\frac{2|E|}{K(K+1)^{2m_A} {\tau_P}^{2m_A}\sqrt{d_ad_b}}\frac{\text{mix}_{y_G}(b, a)}{\gamma}.
\end{equation}
Or, when $s\rightarrow0$, the relation becomes:
\begin{equation}
    L_{m_\Psi} \approx \frac{P}{K}L_{m_A}+\frac{2|E|}{K\sqrt{d_ad_b}}\frac{\text{mix}_{y_G}(b, a)}{\gamma}\cdot\frac{1}{(\alpha^2s^{2K})^{m_\Psi}}.
\end{equation}
\end{proof}

\section{Details of Encoding Eigenvalues}
\label{ee_}
In this paper, we adopt Eigenvalue Encoding (EE) Module~\citep{specformer} to encode eigenvalues. EE functions as a set-to-set spectral filter, enabling interactions between eigenvalues. In EE, both magnitudes and relative differences of all eigenvalues are leveraged. Specifically, the authors use an eigenvalue encoding function to transform each $\lambda$ from scalar $\mathbb{R}^1$ to a vector $\mathbb{R}^d$:
\begin{equation}
        \rho(\lambda, 2i)=\sin{(\epsilon\lambda/10000^{2i/d})},\quad \rho(\lambda, 2i+1)=\cos{(\epsilon\lambda/10000^{2i/d})},
\end{equation}
where $i$ is the dimension of the representations and $\epsilon$ is a hyper parameter. By encoding in this way, relative frequency shifts between eigenvalues are captured. Then, the raw representations of eigenvalues are the concatenation between eigenvalues and corresponding representation vectors:
\begin{equation}
    \bm{Z}_\lambda=[\lambda_1||\rho(\lambda_1),\dots, \lambda_{N-1}||\rho(\lambda_{N-1})]^\top\in\mathbb{R}^{N\times d}.
\end{equation}
To capture the dependencies between eigenvalues, a standard Transformer is used followed by skip-connection and feed forward network (FFN):
\begin{equation}
    \hat{\bm{Z}}_\lambda=\textrm{Transformer}(\textrm{LN}(\bm{Z}_\lambda))+\bm{Z}_\lambda\in\mathbb{R}^{N\times d},\quad \bm{Z}=\textrm{FFN}(\textrm{LN}(\bm{\hat{\bm{Z}}}_\lambda))+\hat{\bm{Z}}_\lambda\in\mathbb{R}^{N\times d},
\end{equation}
where $\textrm{LN}$ is the layer normalization. Then, $\bm{Z}$ is the embedding matrix for eigenvalues, which is injected into the learning of combination coefficients $\tilde{\bm{a}}$ and $\tilde{\bm{b}}$, and scales $\tilde{\bm{s}}$.
\section{Experimental Details}
\label{ajsoifbfnd}
\subsection{Implementation Details}
\label{Implementation details}


\begin{wrapfigure}[16]{r}{0.4\textwidth}
    \centering
  \includegraphics[scale=0.3]{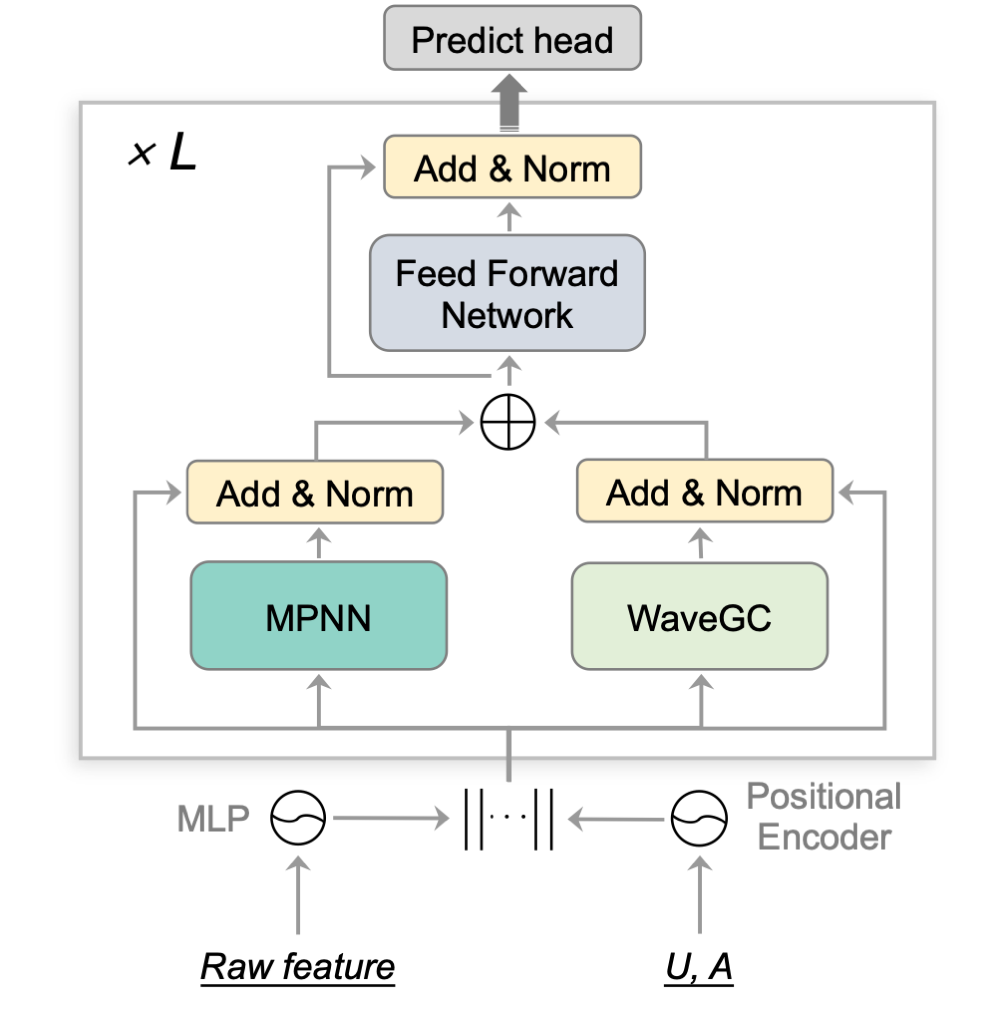}
    \vspace{-0.25cm}
  \caption{Combing MPNN with \M.}
  \label{mpgnn_wave}
  \vskip -0.15in
\end{wrapfigure}
Inspired by~\cite{gps}, we adopt the hybrid network architecture as shown in Fig.~\ref{mpgnn_wave}, where the "\M" block is the process shown in Fig.~\ref{model} (a). This architecture explicitly involves a parallel massage passing neural network (MPNN) (e.g., GCN~\cite{gcn}, GatedGCN~\cite{gated}) to augment the low-frequency modeling. Then, these two branches separately go through skip-connection and normalization, and then sum together followed by a two-layers MLP, eventually skip-connection and normalization.

We explore the number of truncated terms $\rho$ from 1 to 10 and adjust the number of scales $J$ from 1 to 5. Additionally, for the pre-defined vector $\overline{\bm{s}}$ controlling the amplitudes of scales, we test each element in $\overline{\bm{s}}$ from 0.1 to 10.
The usage of the tight frames constraint is also a parameter subject to tuning, contingent on the given dataset. Typically, models iterate through several layers to produce a single result, thus the parameters of \M\ may or may not be shared between different layers.
For short-range datasets, we only retain the first 30\% of eigenvalues and their corresponding eigenvectors for efficient eigendecomposition, and set a threshold $\aleph$ and filter out entries in $\Phi$ and $\Psi_{s_j}$ whose absolute value is lower than $\aleph$.


For fair comparisons, we randomly run 4 times on long-range datasets~\citep{lrgb}, and 10 times on short-range datasets~\citep{nagphormer}, and report the average results with their standard deviation for all methods. For the sake of reproducibility, we also report the related parameters in Appendix~\ref{Hyper-parameters settings}.

\subsection{Datasets Description}

\label{Datasets and baselines}
\begin{table}[h]
  \centering
  \caption{The statistics of the short-range datasets.}
  \label{statistics_node}
  \vskip 0.05in
  \begin{tabular}{c|ccccc}
    \toprule[1.2pt]
    Dataset & \# Graphs & \# Nodes & \# Edges & \# Features & \# Classes \\
    \hline
    CS &1& 18,333& 163,788& 6,805& 15\\
    Photo&1& 7,650& 238,163& 745& 8\\
    Computer&1& 13,752& 491,722& 767& 10\\
    CoraFull&1& 19,793& 126,842& 8,710& 70\\
    ogbn-arxiv&1& 169,343& 1,116,243& 128& 40\\
    \toprule[1.2pt]
\end{tabular}
\vskip -0.1in
\end{table}
For short-range datasets, we choose five commonly used \texttt{CS}, \texttt{Photo}, \texttt{Computer}, \texttt{CoraFull}~\citep{pyg} and \texttt{ogbn-arxiv}~\citep{ogb}. \texttt{CS} is a network based on co-authorship, with nodes representing authors and edges symbolizing collaboration between them. In the \texttt{Photo} and \texttt{Computer} networks, nodes stand for items, and edges suggest that the connected items are often purchased together, forming co-purchase networks. \texttt{CoraFull} is a network focused on citations, where nodes are papers and edges indicate citation connections between them. \texttt{ogbn-arxiv} is a citation network among all Computer Science (CS) Arxiv papers, where each node corresponds to an Arxiv paper, and the edges indicate the citations between papers. The details of these five datasets are summarized in Table~\ref{statistics_node}.

\begin{table}[h]
  \centering
  \caption{The statistics of the long-range datasets.}
  \label{statistics_graph}
  \vskip 0.05in
  \resizebox{\textwidth}{!}{
  \begin{tabular}{c|cccccc}
    \toprule[1.2pt]
    Dataset & \# Graphs & Avg. \# nodes & Avg. \# edges & Prediction level & Task & Metric\\
    \hline
    PascalVOC-SP& 11,355& 479.4& 2,710.5& inductive node& 21-class classif.& F1 score\\
    PCQM-Contact& 529,434& 30.1& 61.0& inductive link& link ranking& MRR\\
    COCO-SP& 123,286& 476.9& 2,693.7& inductive node& 81-class classif.& F1 score\\
    Peptides-func& 15,535& 150.9& 307.3& graph& 10-task classif.& Avg. Precision\\
    Peptides-struct& 15,535& 150.9& 307.3& graph& 11-task regression& Mean Abs. Error\\
    \toprule[1.2pt]
\end{tabular}}
\vskip -0.1in
\end{table}
For long-range tasks, we choose five long-range datasets~\citep{lrgb}, including \texttt{PascalVOC-SP (VOC)}, \texttt{PCQM-Contact (PCQM)}, \texttt{COCO-SP(COCO)}, \texttt{Peptides-func (Pf)} and \texttt{Peptides-struct (Ps)}. These five datasets are usually used to test the performance on long-range modeling. \texttt{VOC} and \texttt{COCO} datasets are created through SLIC superpixelization of the Pascal VOC and MS COCO image collections. They are both utilized for node classification, where each super-pixel node is categorized into a specific object class. \texttt{PCQM} is developed from PCQM4Mv2~\citep{hu2021ogb} and its related 3D molecular structures, focusing on binary link prediction. This involves identifying node pairs that are in 3D contact but distant in the 2D graph. Both \texttt{Pf} and \texttt{Ps} datasets consist of atomic graphs of peptides sourced from SATPdb. In the \texttt{Peptides-func} dataset, the task involves multi-label graph classification into 10 distinct peptide functional classes. Conversely, the \texttt{Peptides-struct} dataset is centered on graph regression to predict 11 different 3D structural properties of peptides. The details of these five datasets are summarized in Table~\ref{statistics_graph}.

\subsection{More analyses for section~\ref{The effectiveness of general wavelet bases}}
In this section, we firstly give a further visualization on short-range dataset and then analyze the impact of the learned scales.

\subsubsection{Visualization on CoraFull}
\label{visualization on}
\begin{figure}[h]
\centering
\begin{minipage}[t]{0.48\textwidth}
\centering
\includegraphics[scale=0.23]{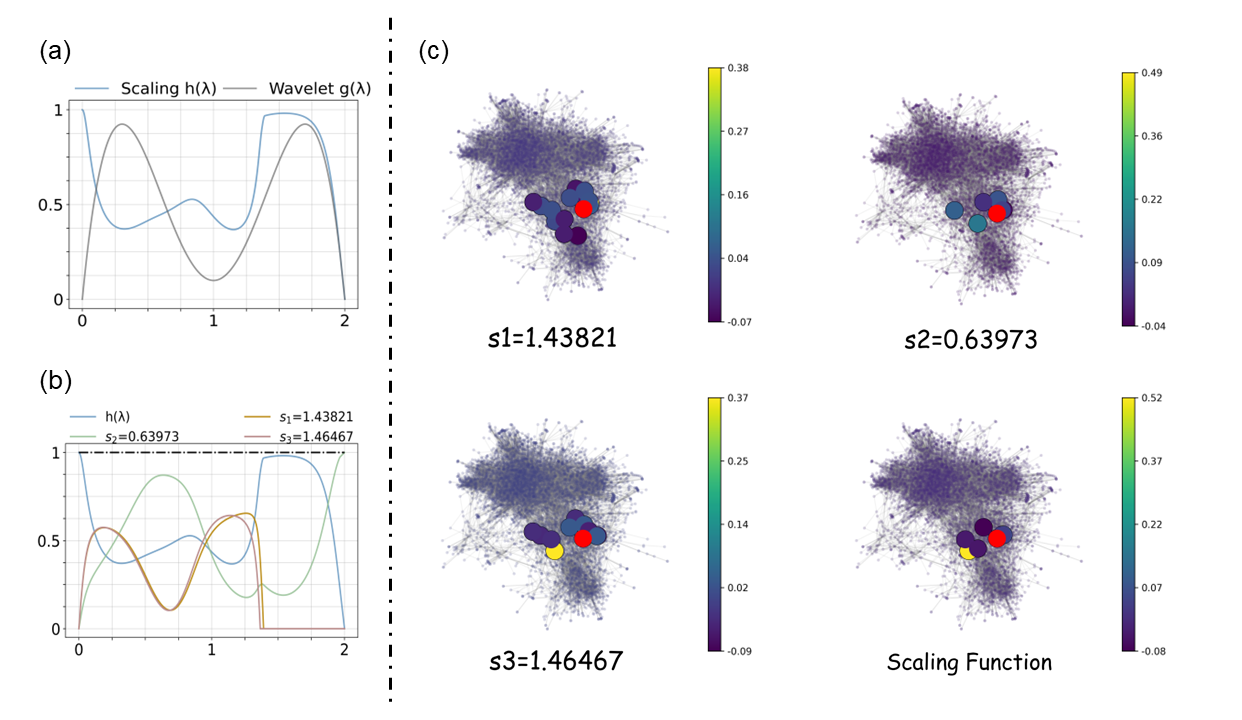}
\caption{Illustration of the spectral and spatial signals of the learned function basis and multiple wavelet bases with full spectrum.}
\label{asfa}
\end{minipage}
\hspace{10pt}
\begin{minipage}[t]{0.48\textwidth}
\centering
\includegraphics[scale=0.23]{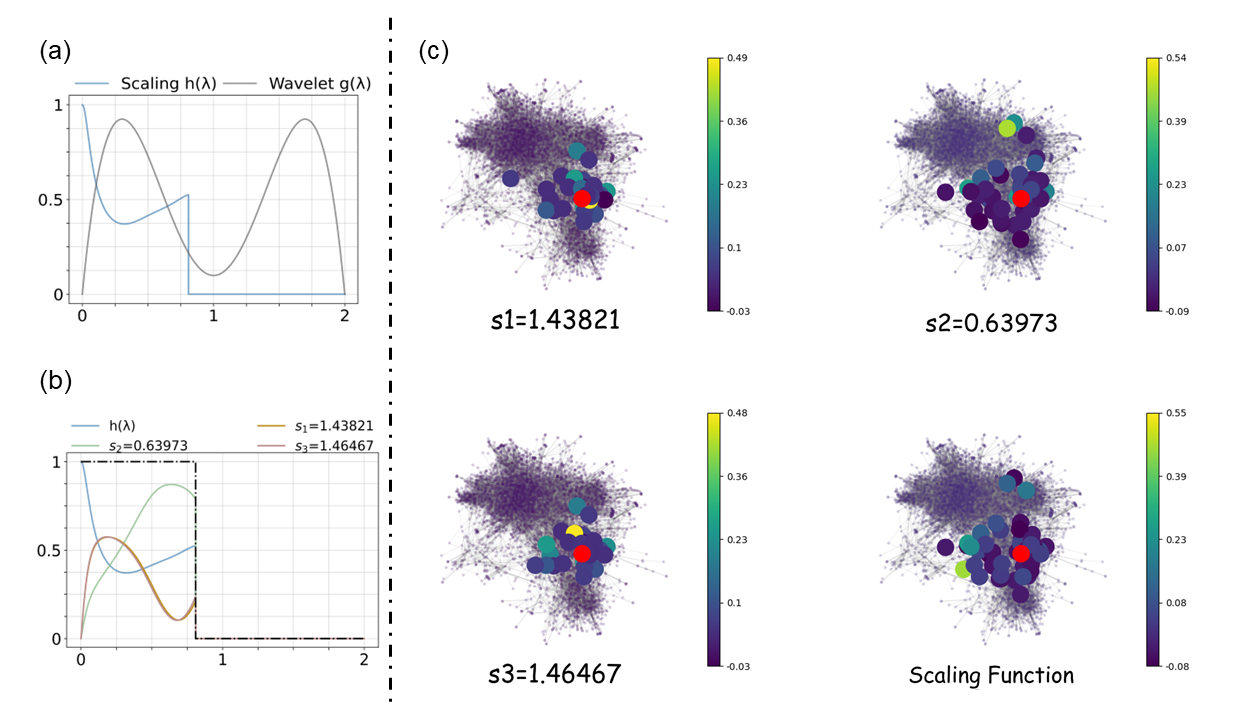}
\caption{Illustration of the spectral and spatial signals of the learned function basis and multiple wavelet bases with partial spectrum.}
\label{adgherhg}
\end{minipage}
\end{figure}
To give one more example, we provide additional visualization results on the CoraFull dataset. These results are presented in Fig.~\ref{asfa}, where the learned scaling functions $h(\lambda)$ and $g(\lambda)$ meet the specified requirements. The four subfigures in Fig.~\ref{asfa}(c) illustrate that as the scale $s_j$ increases, the receptive field of the center node expands. This highlights \M's capability to capture both short- and long-range information by adjusting different values of $s_j$. However, one of our strategies for CoraFull involves considering only 30\% of eigenvalues as input. Consequently, the full spectrum is truncated, leaving only the remaining 30\% parts, as depicted in Fig.~\ref{adgherhg}. We give a deeper insight in the behavior of this truncation from both spectral and spatial perspectives:\
\begin{itemize}
    \item \textbf{Spectral perspective}. As shown in Fig.~\ref{adgherhg}, the wavelet function $g(s\lambda)$ retains non-trivial amplitudes within the first 30\% domain. While $g(\lambda)\approx0$, the retained spectral range is sufficiently broad to allow the wavelets to operate effectively. Therefore, even in the truncated setting, both the scaling function $h(\lambda)$ and wavelet function $g(\lambda)$ contribute meaningfully to low-frequency modeling.
    \item \textbf{Spatial perspective}. In Fig.~\ref{adgherhg}(b), we observe that truncating the spectrum mimics the effect of using a larger wavelet scale $s>1$, which reduces the effective spectral range and increases the spatial receptive field. This effect is visually confirmed in Fig.~\ref{asfa}(c) and~\ref{adgherhg}(c), where the receptive fields become noticeably larger after truncation. Thus, even on short-range datasets, the wavelet branch captures valuable higher-order information that complements local aggregation from MPNN. This complementary role is further validated by the performance drop observed in Table~\ref{abla} when wavelets are removed.
\end{itemize}
Overall, spectral truncation does not impair wavelet behavior; instead, it supports effective low-frequency modeling while also enhancing spatial coverage.

\subsubsection{Impact of the learned scales}
\begin{figure}[htbp]
\centering
\begin{minipage}{0.48\textwidth}
    \centering
    \includegraphics[scale=0.22]{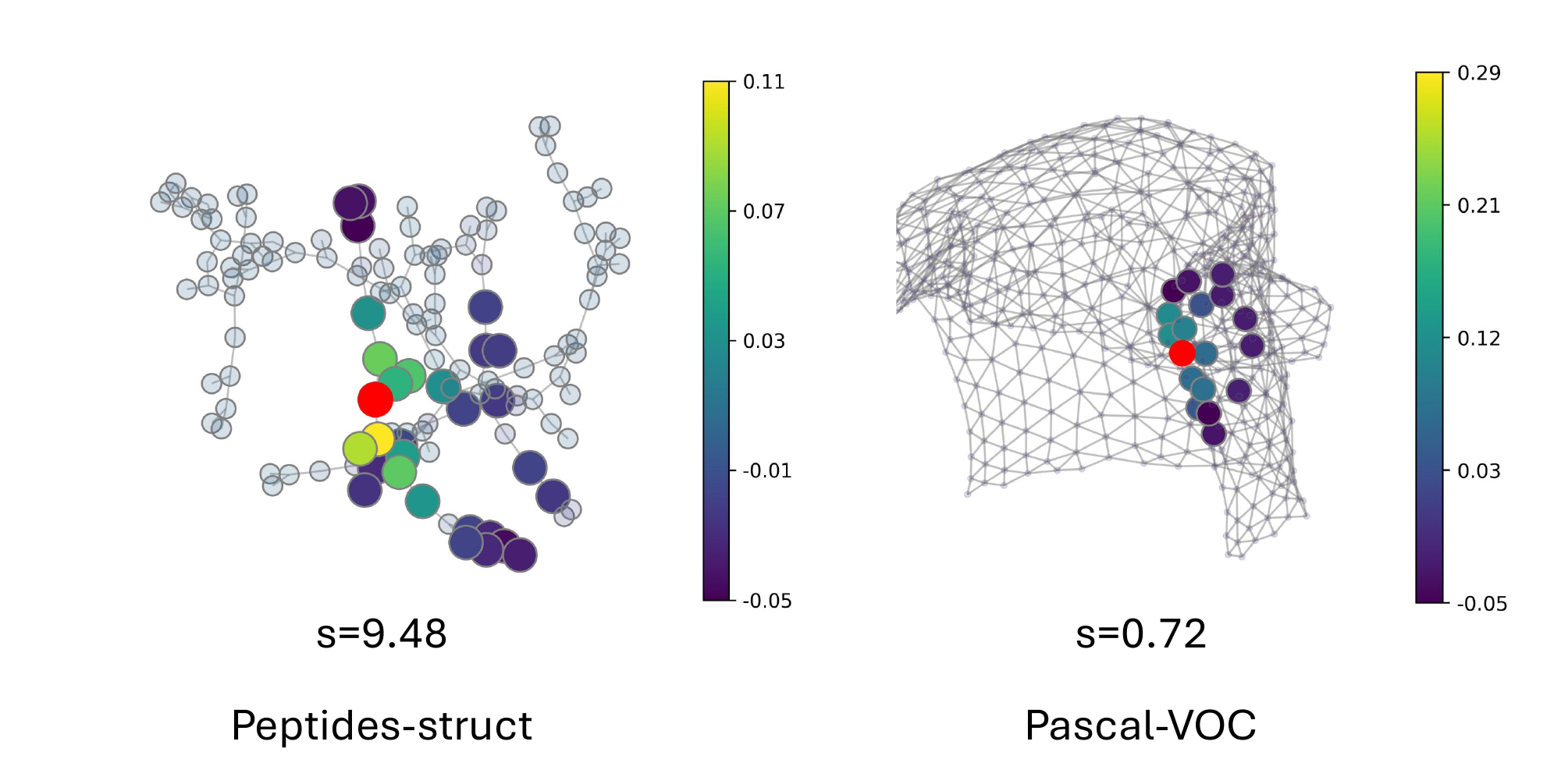} 
    \caption{Visualizations of receptive fields for Peptides-struct (Ps) and Pascal-VOC (VOC) at their largest scale $s$.}
    \label{re_v}
\end{minipage}%
\hfill
\begin{minipage}{0.48\textwidth}
    \centering
    \captionof{table}{Comparison of average and max receptive fields of Ps and VOC.}
  \vskip 0.15in
    \begin{tabular}{ccc}
        \toprule[1pt]
    & Peptides-struct & Pascal-VOC \\
    \hline
    Avg. Receptive Field&	3.02 & 0.74\\
    Max Receptive Field&	9 & 3\\
    Avg. Shortest Path & 20.89 & 10.74\\
    \toprule[1pt]
    \end{tabular}
    \label{re_t}
\end{minipage}
\end{figure}
We empirically analyze the learned scale values and their impact on receptive fields in Fig.~\ref{re_v}. Specifically, we illustrate the largest learned scales for the Ps and VOC datasets, along with their corresponding receptive field visualizations. The receptive field is heuristically defined as follows:
\begin{Definition}(Receptive field.)
    Given a wavelet $\Psi(s\lambda)$, node $j$ lies in the receptive field of node $i$ if $|\Psi(s\lambda)[i, j]| > 0.1 \times \max(|\Psi(s\lambda)|)$.
\end{Definition}
Under this criterion, we observe that Ps exhibits larger receptive fields, corresponding to a larger learned scale of 9.48. We further report the average and maximum receptive field sizes across all nodes in Table~\ref{re_t}. The larger receptive fields in Ps align with its inherently longer average shortest-path distances, thus validating the model’s ability to adaptively adjust to long-range dependencies.
\begin{figure*}[t]
  \centering
  \includegraphics[scale=0.3]{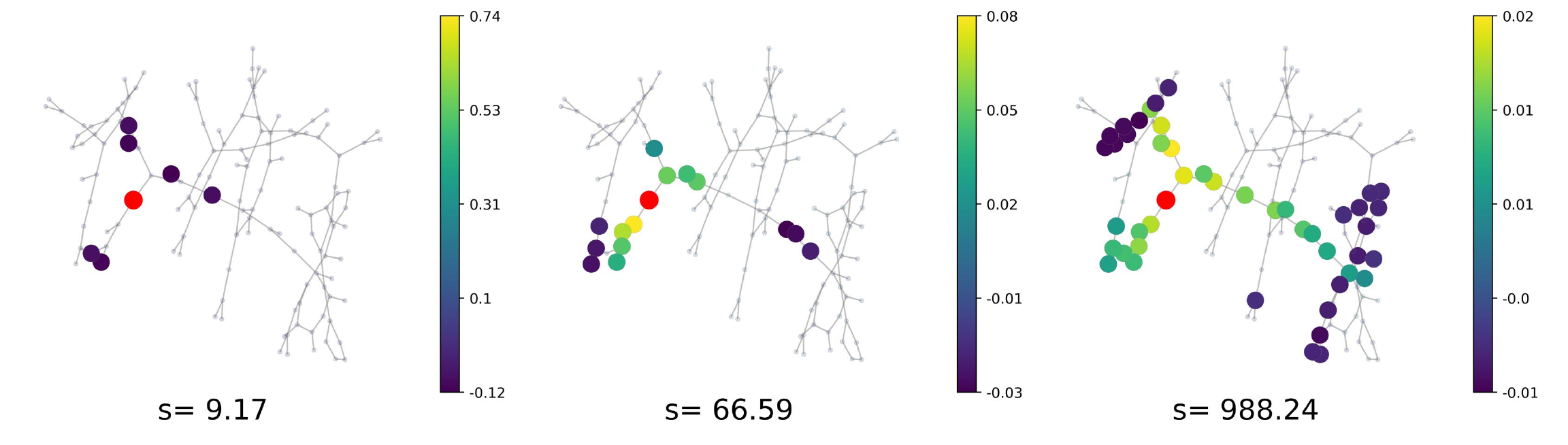}
    \vspace{-0.25cm}
  \caption{Visualizations of receptive fields for Peptides-func (Pf) at extreme scales.}
  \label{extreme}
  \vskip -0.15in
\end{figure*}
To examine the extreme case of large-scale values, we increase the predefined scale vector $\bar{s}$ in Eq.~\eqref{zxvzxv} for Peptides-func (Pf) to (10, 100, 1000). This vector determines the upper bound of the learnable scale range. The resulting learned scales and receptive fields are depicted in Figure 3. When s = 9.17, the red node primarily aggregates local information; in contrast, at s = 988.24, the same node gathers information from a much broader range. This confirms our theoretical assertion that WaveGC exhibits long-range behavior as s approaches infinity.

\subsection{More comparisons between WaveGC and ChebNet}
\label{more comparisons}
\begin{table*}[h]
  \caption{More ablations for differences between WaveGC and ChebNet.}
  \label{1}
  \vskip 0.05in
  \centering
  \setlength{\tabcolsep}{3mm}{
  \begin{tabular}{cccccc}
    \toprule[1pt]
    & Free $\tilde{\alpha}$ & Free $\tilde{\beta}$ & Fix s=1 & Free $\tilde{s}$&Original \\
    \hline
    Computer&	91.28±0.15&	91.19±0.09& 91.73±0.02&	91.51±0.02&	\textbf{92.26±0.18}\\
    Ps&	25.08±0.01&	25.09±0.12&	25.28±0.00 &	25.15±0.25&	\textbf{24.83±0.11}\\
    \toprule[1pt]
  \end{tabular}}
\end{table*}

Obviously, both \M\ and ChebNet attempt weighted combination of Chebyshev polynomials in different ways. On one hand, ChebNet learns term coefficients independently, while \M\ map eigenvectors into coefficients $\tilde{\alpha}$ and $\tilde{\beta}$. On the other hand, \M\ further involve multiple and learnable scales $\tilde{s}$. Finally, we test importance of these differences on the Computer and Ps. The results are summarized in Table~\ref{1}, showcasing different variants such as free learning coefficients (i.e., $\tilde{\pmb \alpha}$, $\tilde{\pmb \beta}$), adopting single scale s=1, and free learning $\tilde{\pmb s}$ to avoid joint parameterization. Each of these modifications resulted in degraded performance compared to the original model, demonstrating the improvements our new model offers over ChebNet.

\subsection{Complexity and Running time}
\label{Time and space complexity analysis}
\begin{table*}[h]
  \caption{Comparison on running time per epoch (s).}
  \label{3}
  \vskip 0.05in
  \centering
  \resizebox{0.9\textwidth}{!}{
  \begin{tabular}{cccccccc|c}
    \toprule[1.3pt]
     & SGWT & GWNN & WaveShrink & WaveNet & DEFT & UFGConvS & UFGConvR& \M \\
    \hline
    Computer &3.6 &0.8 &2.7 &2.0 &1.1 &3.1 &3.2 &1.5\\
    Ps &21.0 &30.5 &52.0 &27.3 &23.6 &47.3 &43.5 &23.9\\
    \toprule[1.3pt]
  \end{tabular}}
\end{table*}
The main contribution of \M\ is to address long-range interactions in graph convolution, so it inevitably establishes spatial connections between distant nodes. This results in the same $O(N^2)$ complexity as Transformer~\citep{va_trans}. This is the same for all spectral graph wavelets, including SGWT, GWNN, WaveShrink, WaveNet and UFGConvS/R. A possible solution is to decrease the number of considered frequency modes from $N$ to $\nu$. In this way, the complexity is reduced to $O(\nu\cdot N)$. We report the running time consumption of \M\ and other spectral graph wavelets (that is, SGWT, GWNN, WaveShrink, WaveNet, DEFT, UFGConvS and UFGConvR). The time consumptions for \texttt{Computer} and \texttt{Ps} are presented in Table~\ref{3}. According to the table, the running time of WaveGC is in the first level among spectral graph wavelets.

\subsection{Hyper-Parameter Sensitivity Analysis}\label{hperpara}
In \M, two key hyper-parameters, namely $\rho$ and $J$, play important roles. The parameter $\rho$ governs the number of truncated terms for both ${T_i^o}$ and ${T_i^e}$, while $J$ determines the number of scales $s_j$ in Eq.~\eqref{zxvzxv}. In this section, we explore the sensitivity of $\rho$ and $J$ on the Peptides-struct (Ps) and Computer datasets. The results are visually presented in Fig.\ref{hyper}, where the color depth of each point reflects the corresponding performance (the lighter the color, the better the performance), and the best points are identified with a red star.
Observing the results, we note that the optimal value for $\rho$ is 2 for Ps and 7 for Computer. This discrepancy can be attributed to the substantial difference in the graph sizes between the two datasets, with Computer exhibiting a significantly larger graph size (refer to Appendix~\ref{Datasets and baselines}). Consequently, a more intricate filter design is necessary for the larger dataset. Concerning $J$, the optimal value is determined to be 3 for both Ps and Computer. A too small $J$ leads to inadequate coverage of ranges, while an excessively large $J$ results in redundant scales with overlapping ranges.
\begin{figure}[h]
\centering
\subfigure[Ps: $\rho$-analysis]{
\label{f_ps}
\includegraphics[scale=0.23]{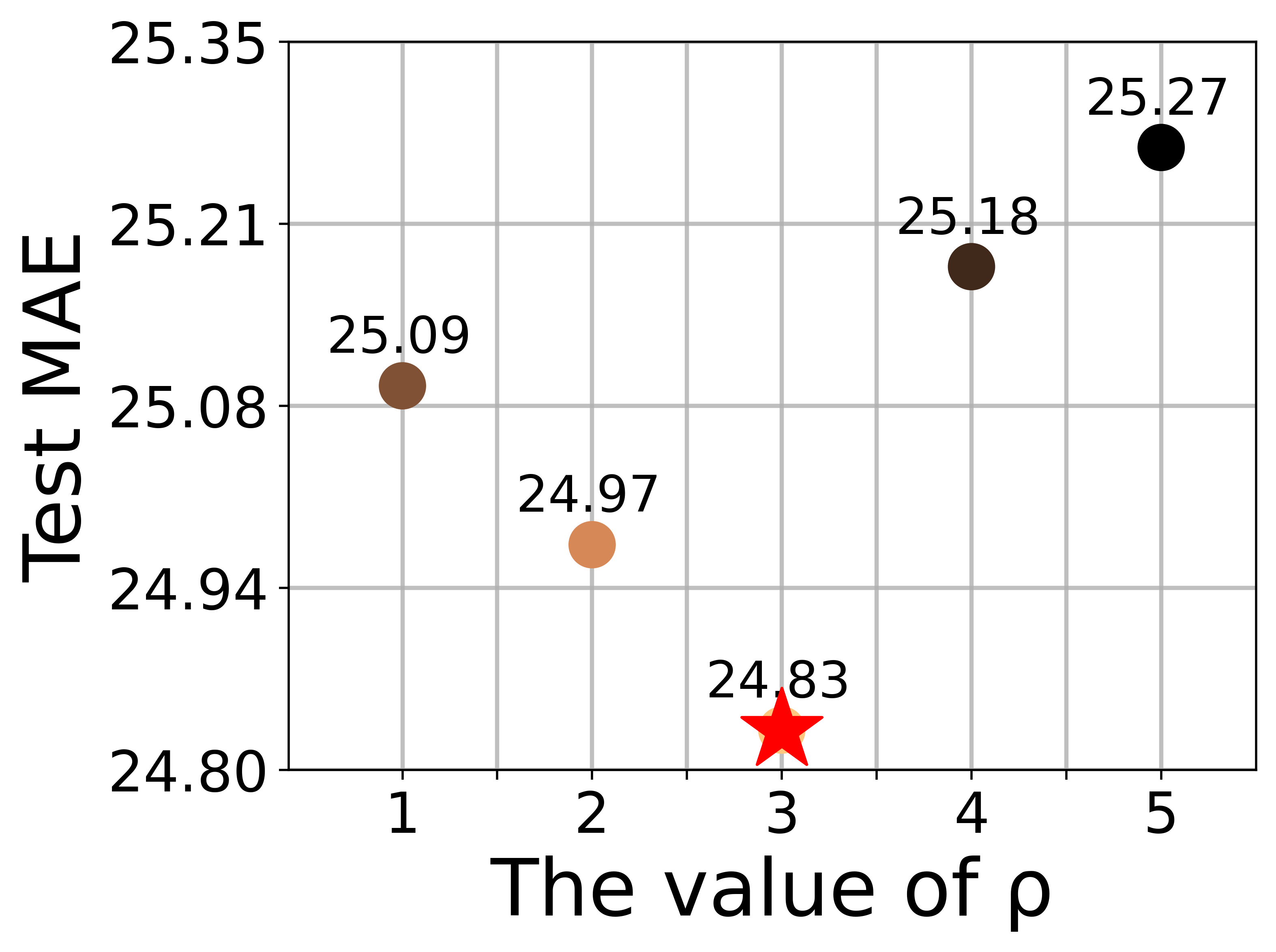}
}
\subfigure[Computer: $\rho$-analysis]{
\label{f_com}
\includegraphics[scale=0.23]{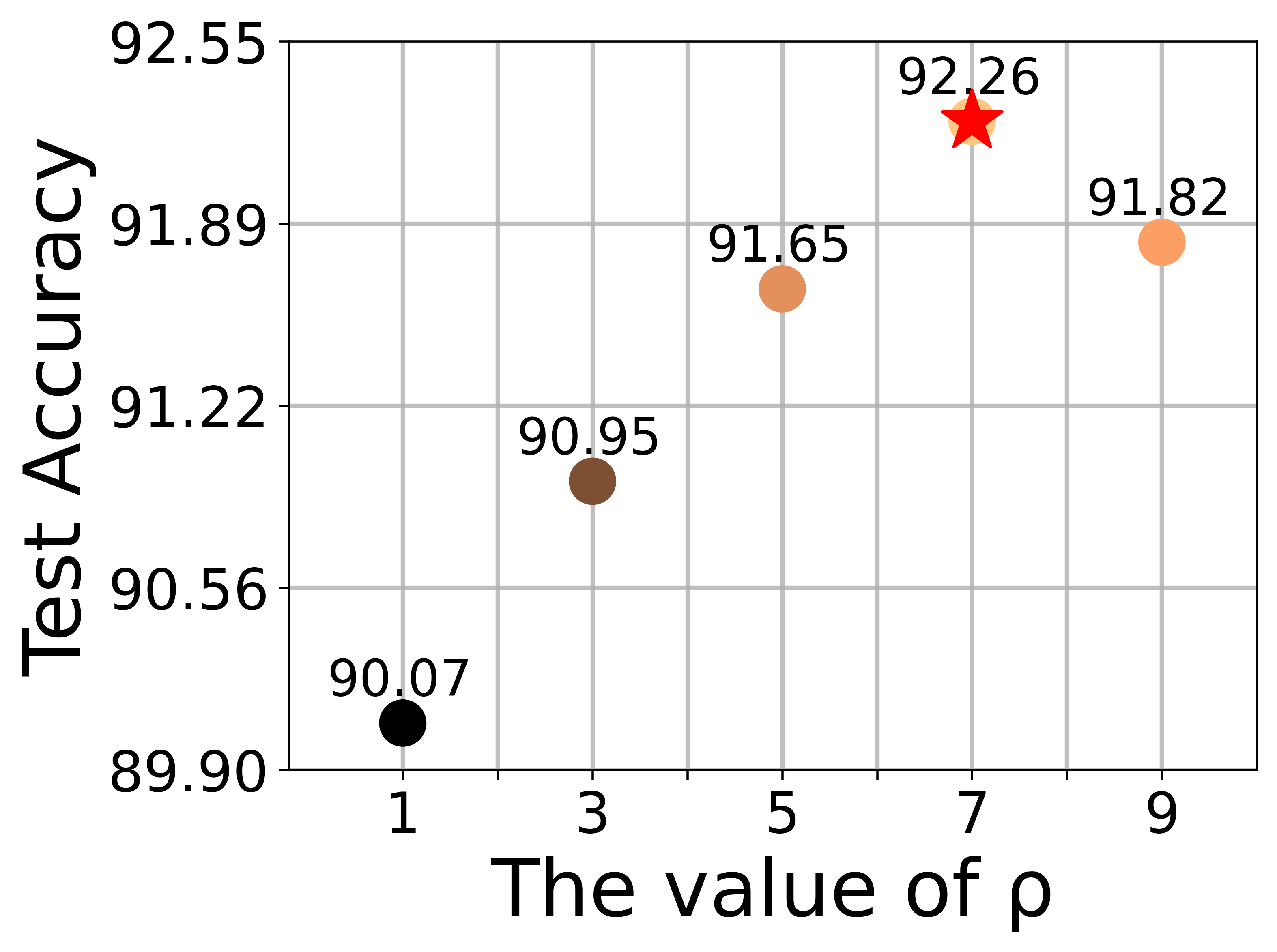}
}
\subfigure[Ps: $J$-analysis]{
\label{j_ps}
\includegraphics[scale=0.23]{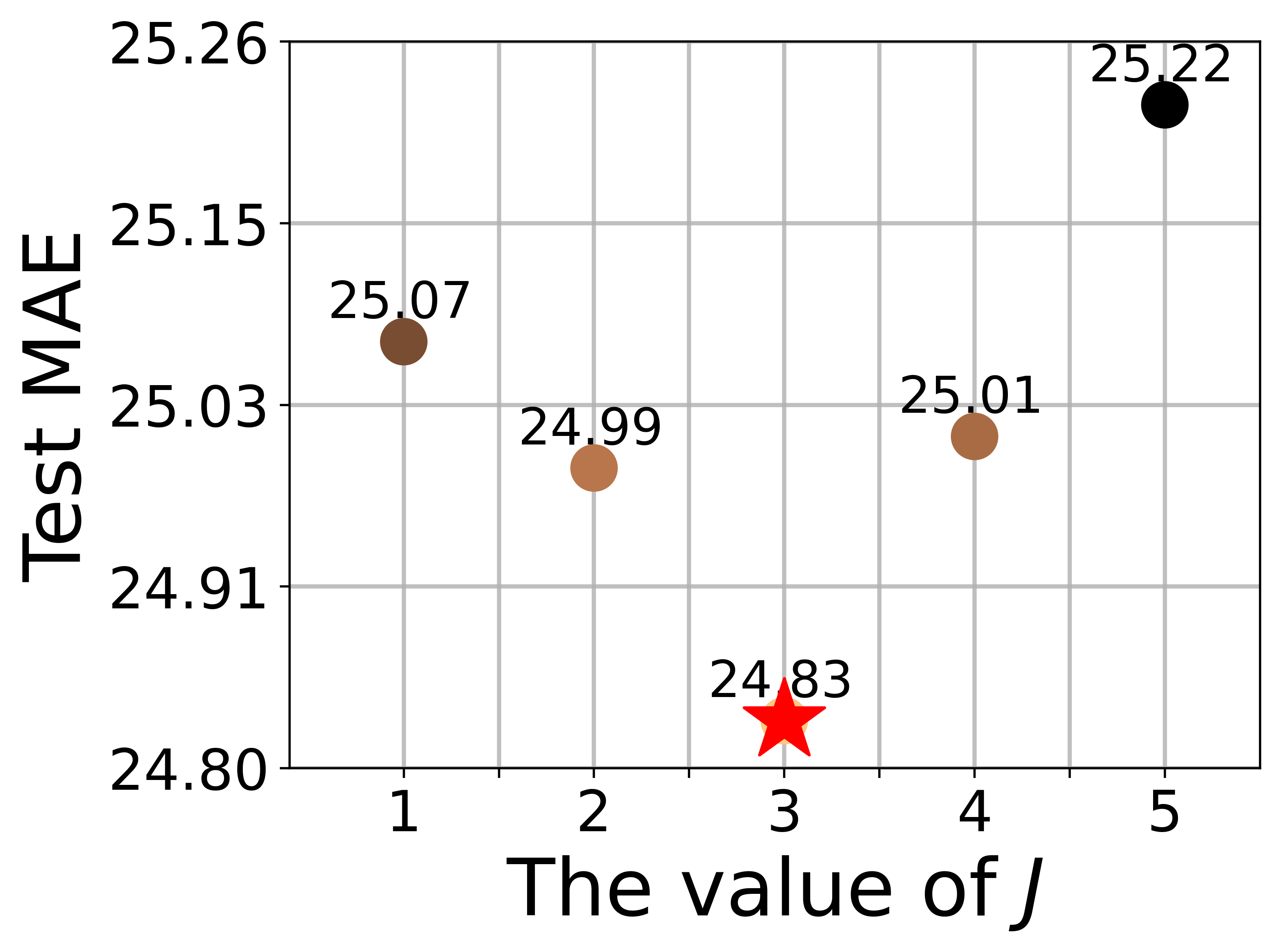}
}
\subfigure[Computer: $J$-analysis]{
\label{j_com}
\includegraphics[scale=0.23]{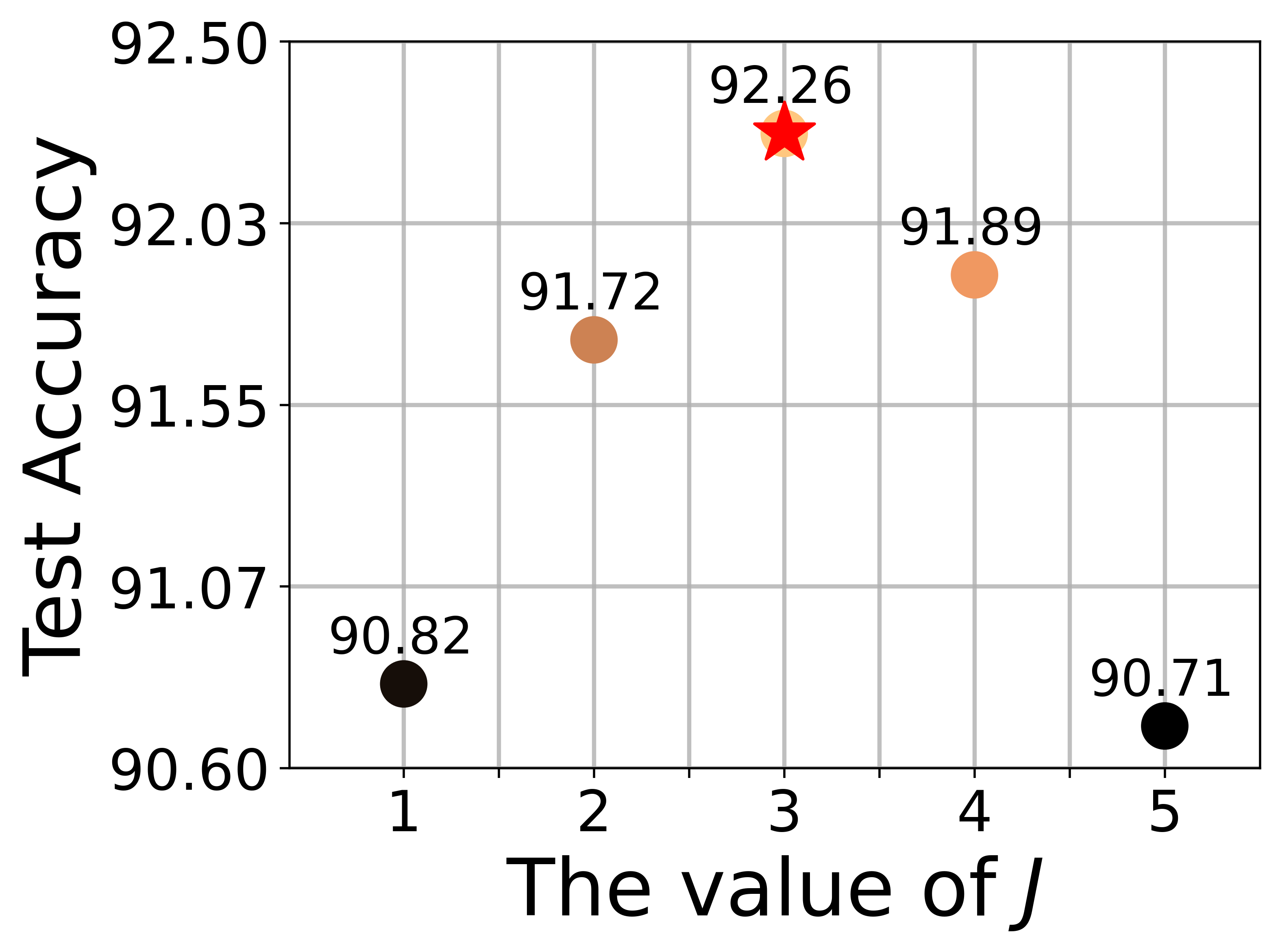}
}
\vspace{-0.25cm}
\caption{Analysis of the sensitivities of $\rho$ and $J$.}
\label{hyper}
\end{figure}

\subsection{Hyper-parameters Settings}
\label{Hyper-parameters settings}
We implement our \M\ in PyTorch, and list some important parameter values in our model in Table~\ref{para_1}. Please note that for the five long-range datasets, we follow the parameter budget $\sim$500k~\citep{lrgb}.
\begin{table}[h]
  \caption{The values of parameters used in \M\ (T: True; F: False).}
  \label{para_1}
  \centering
  \vskip 0.05in
  \setlength{\tabcolsep}{3mm}{
  \begin{tabular}{c|cccccc}
        \bottomrule
       Dataset & \# parameters &$\rho$ & $J$ & $\overline{\bm{s}}$ & Tight frames & $\aleph$ \\
        \bottomrule
       CS & 495k&3 & 3 & \{0.5, 0.5, 0.5\} & T & 0.1 \\
       Photo & 136k&3 & 3 & \{1.0, 1.0, 1.0\} & T & 0.1 \\
       Computer & 167k&7 & 3 & \{10.0, 10.0, 10.0\} & T & 0.1\\
       CoraFull & 621k&3 & 3 & \{2.0, 2.0, 2.0\} & T & 0.1\\
       ogbn-arxiv & 2,354k&3 & 3 & \{5.0, 5.0, 5.0\} & F & /\\
       \hline
       PascalVOC-SP & 506k&5 & 3 & \{0.5, 1.0, 10.0\} & T & / \\
       PCQM-Contact & 508k&5 & 3 & \{0.5, 1.0, 5.0\} & T & / \\
       COCO-SP & 546k&3  & 3 & \{0.5, 1.0, 10.0\} & T & /\\
       Peptides-func & 496k&5 & 3 & \{10.0, 10.0, 10.0\} & T & /\\
       Peptides-struct & 534k&3 & 3 & \{10.0, 10.0, 10.0\} & F & /\\
        \bottomrule
        
  \end{tabular}
  }
  \vskip -0.1in
\end{table}

\subsection{Operating Environment}
\label{Operating environment}
The environment where our code runs is shown as follows:
\begin{itemize}
    \item Operating system: Linux version 5.11.0-43-generic
    \item CPU information: Intel(R) Xeon(R) Gold 6226R CPU @ 2.90GHz
    \item GPU information: NVIDIA RTX A5000
\end{itemize}

\section{Approximation Strategy for O(N) Complexity}
\label{ssssssss}
\begin{figure}[h]
  \centering
  \begin{minipage}{0.48\textwidth}
    \centering
    \includegraphics[scale=0.26]{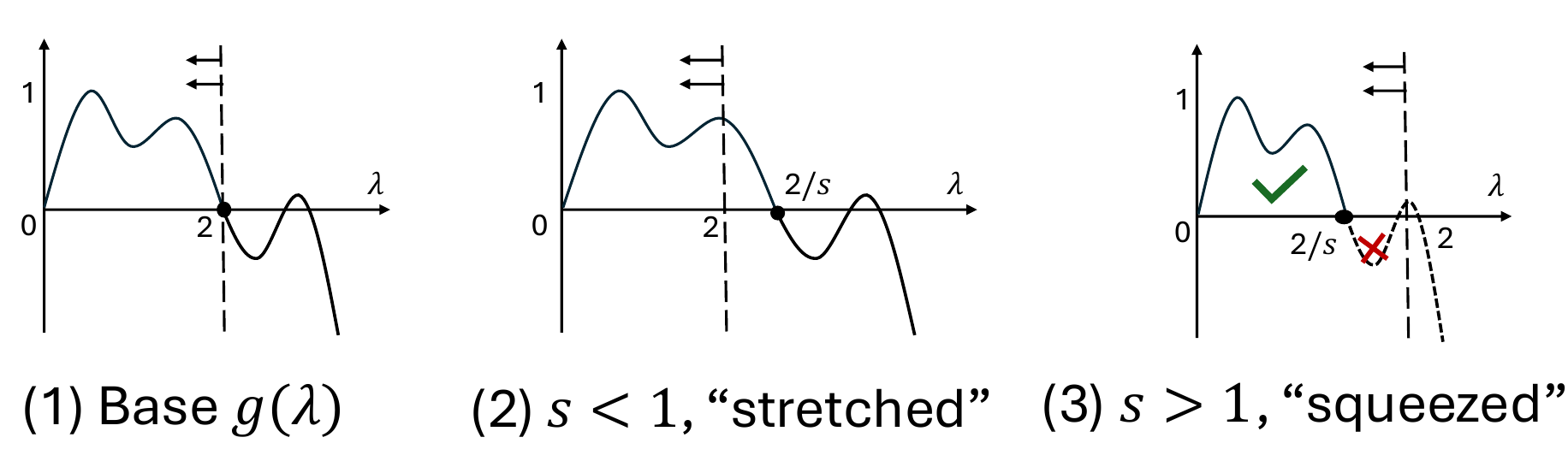}
    \caption{The scale $s$ can “stretch” or “squeeze” the shape of $g(\lambda)$ as $g(s\lambda)$.}
    \label{scalesssss}
  \end{minipage}
  \hfill
  \begin{minipage}{0.48\textwidth}
    \centering
    \includegraphics[scale=0.26]{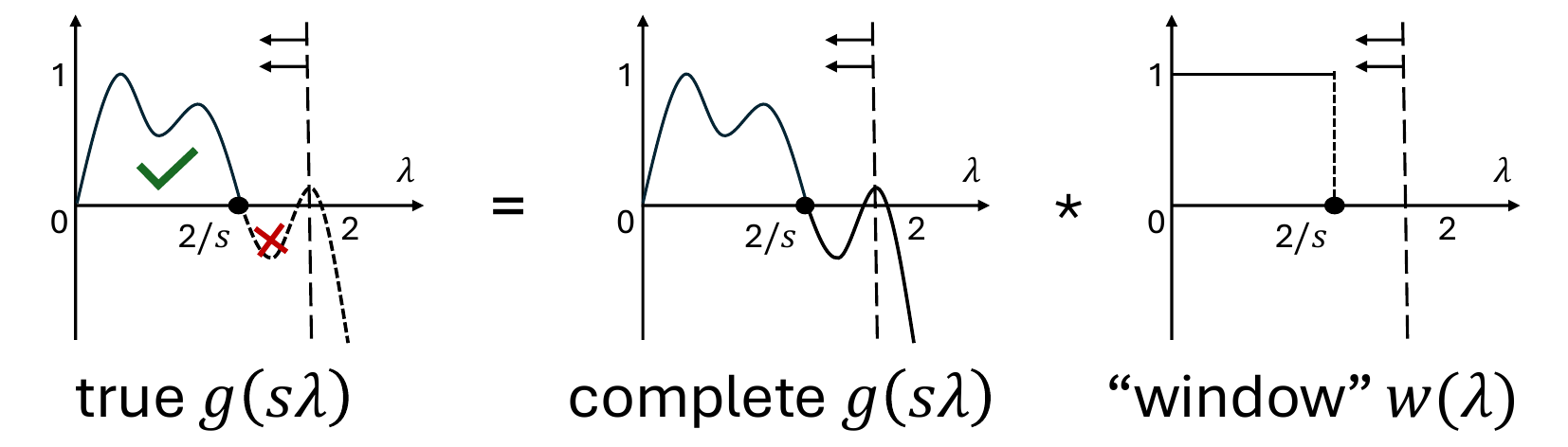}
    \caption{The illustration of applying “window” over $g(s\lambda)$.}
    \label{window}
  \end{minipage}
\end{figure}
To further reduce complexity, we propose a fully polynomial-based approximation that removes the need for eigendecomposition, achieving total complexity of O(N), on par with graph Fourier-basis-based methods. This is achieved via polynomial approximation of the wavelet transform using Chebyshev polynomials:
\begin{itemize}
    \item \textbf{Scaling function $h(\Lambda)$.} Since $h(\bm{\Lambda})=\sum b_iT^o_i(\bm{\Lambda})$, where $T^o_i$ are odd-degree Chebyshev polynomials, we can compute
    \begin{equation}
        \bm{\Phi} \bm{f}=\bm{U}h(\bm{\Lambda})\bm{U}^\top \bm{f}=\sum b_iT_i^o(\bm{L})\bm{f}.
    \end{equation}
    This is equivalent to a polynomial operation over the graph Laplacian $\bm{L}$, which has $O(N)$ complexity.
    \item \textbf{Wavelet Function $g(\Lambda)$.} Similarly, $g(\bm{\Lambda})=\sum a_iT^e_i(\bm{\Lambda})$, where $T^e_i$ are even-degree Chebyshev polynomials, gives
    \begin{equation}
        \bm{\Psi} \bm{f}=\bm{U}g(\bm{\Lambda})\bm{U}^\top \bm{f}=\sum a_iT_i^e(\bm{L})\bm{f},
    \end{equation}
    which is also polynomial in $\bm{L}$ with $O(N)$ cost.
    \item \textbf{Incorporating scale $s$.} The domain $\lambda\in[0,2]$ for $g(\lambda)$ transforms to $\lambda\in[0,2/s]$ in $g(s\lambda)$. This raises two scenarios:
    \begin{itemize}
        \item If $s<1$: The full spectrum [0,2] is covered, and $g(s\lambda)$ remains valid as a polynomial (Fig.~\ref{scalesssss} (2)).
        \item If $s>1$: Only the interval [0,2/s] is valid. The rest of the spectrum [2/s,2] should be suppressed (Fig.~\ref{scalesssss} (3)). To handle this, we apply a \textit{window function} $w(\lambda)$, where:
        \[
w(\lambda) =
\begin{cases}
1 &\quad\lambda\in[0,2/s]\\
0 &\quad\lambda\in[2/s, 2]\\
\end{cases}
\]
The true scaled wavelet becomes $g(s\lambda)=g(s\lambda)\cdot w(\lambda)$. Both $g(s\lambda)$ and $w(\lambda)$ can be approximated using Chebyshev polynomials, so the entire operation remains within $O(N)$ complexity.
    \end{itemize}
\end{itemize}
Using this approach, plus without eigenvalue encoding (EE) and tight frame constraint, we no longer require EVD with the maximum simplification. The resulting model maintains the theoretical structure of WaveGC while gaining substantial computational benefits.

\begin{table}[h]
  \centering
  \begin{minipage}{0.48\textwidth}
    \centering
    \caption{Running time (s) per epoch.}
    \label{ttttt}
    \begin{tabular}{cccc}
    \toprule[1pt]
       & CS& Photo& Computer \\
      \hline
      GPRGNN& 1.1& 0.2& 0.4\\
    BernNet& 1.5& 0.5& 1.3\\
    UniFilter& 5.7& 0.8& 1.5\\
    \hline
    \M\_simplified& 1.4& 0.6& 1.8\\
    \toprule[1pt]
    \end{tabular}
  \end{minipage}
  \hfill
  \begin{minipage}{0.48\textwidth}
    \centering
    \caption{Qualified results on three short-range datasets.}
    \label{rrrrr}
    \begin{tabular}{cccc}
    \toprule[1pt]
       Accuracy $\uparrow$& CS & Photo& Computer \\
      \hline
      GPRGNN        & 95.13& 94.49& 90.82 \\
      BernNet       & 95.42& 94.67& 90.98 \\
      UniFilter     & 95.68& 94.34& 90.07 \\
      \hline
      \M\_simplified& 95.63& 94.90& 91.22\\
      \M &95.89 &95.37 &92.26\\
    \toprule[1pt]
    \end{tabular}
  \end{minipage}
\end{table}

To validate this simplified version, we compared its runtime and accuracy with GPRGNN~\citep{gprgnn}, BernNet~\citep{bern}, and UniFilter~\citep{uni} on three short-range datasets. As shown in Table~\ref{ttttt}, the \textit{\M\_simplified} achieves comparable training time to Fourier-based methods. According to Table 2, it incurs only a small drop in performance compared with \M, confirming that polynomial approximation remains effective even without EVD.
\section{Related Work}
\label{related work}
\textbf{Graph Wavelet Transform.} Graph wavelet transform is a generalization of classical wavelet transform~\citep{cw} into graph domain. SGWT~\citep{gw} defines the computing paradigm on weighted graph via spectral graph theory. Specifically, it defines scaling operation in time field as the scaling on eigenvalues. The authors also prove the localization properties of SGWT in the spatial domain in the limit of fine scales. To accelerate the computation on transform, they additionally present a fast Chebyshev polynomial approximation algorithm. GWNN~\citep{gwnn} chooses heat kernel as the filter to construct the bases. The graph wavelet bases learnt from these methods are not guaranteed as band-pass filters in $\lambda\in[0, 2]$ and thus violate admissibility condition~\citep{cw}. 
UFGCONV~\citep{ufg} defines a framelet-based graph convolution with Haar-type filters. WaveNet~\citep{yang2024wavenet} relies on Haar wavelets as bases, and uses the highest-order scaling function to approximate all the other wavelets and scaling functions. WGGP~\citep{opolka2022adaptive} integrates Gaussian processes with Mexican Hat to represent varying levels of smoothness on the graph.
The above four methods fix the form of the constructed wavelets, extremely limiting the adaptivity to different datasets. 
In this paper, our \M\ constructs band-pass filter and low-pass filter purely depending on the even terms and odd terms of Chebyshev polynomials. In this case, the admissibility condition is strictly guaranteed, and the constructed graph wavelets can be arbitrarily complex and flexible with the number of truncated terms increasing. In addition, SEA-GWNN~\citep{deb2024sea} focuses on the second generation of wavelets, or lifting schemes, which is a different topic from ours.

\textbf{Graph Scattering Transform.} The Scattering Transform constructs a hierarchical, tree-like structure by combining a cascading filter bank (or wavelets), point-wise non-linearity, and a low-pass operator. As introduced by Mallat~\cite{mallat2012group}, this approach guarantees translation invariance and stability to deformations. On one hand, researchers have explored the application of this technique to graph data. Early efforts, such as those by~\cite{zou2020graph},~\cite{gama2019stability}, and GS-SVM~\cite{gao2019geometric}, extended the scattering transform into the graph spectral domain. ST-GST~\cite{pan2020spatio} defined filtering and wavelets for spatio-temporal graphs, deriving the corresponding scattering process. Meanwhile,\cite{scattering} employed lazy diffusion\cite{coifman2006diffusion} as wavelets to construct graph diffusion scattering, demonstrating its stability against deformations based on diffusion distance. HDS-GNN~\cite{asdsd} enhanced GNNs by integrating scattering features from a diffusion scattering network layer by layer, while GGSN~\cite{koke2022graph} introduced further flexibility to each operation. On the other hand, the computational efficiency of this transform, which involves a total of $\sum_{l=1}^L J^l$ filtering operations, poses a significant challenge. To address this, pGST~\cite{ioannidis2020pruned} proposed a pruning strategy, retaining only the higher-energy child signals for each parent node. Scattering GCN~\cite{Scattering_gcn} further optimized the process by selectively using more beneficial wavelets, simplifying the scattering computation.

\textbf{Spectral graph convolution.} Traditional studies on spectral graph convolution mainly concentrate on the design of filter with fixed Fourier bases. One way is to design low-pass filters that smooth signals within neighboring regions. GCN~\cite{gcn} keeps the first two ChebNet~\cite{chebnet} terms with extra tricks, and averages signals between neighbors. PPNP~\cite{appnp} smooths signals in a broader range following PageRank based diffusion. Another way is to design adaptive filters so work in both homophily and heterophily scenarios. ChebNet~\cite{chebnet} approximates universe filter functions with learnable coefficients before each Chebyshev term. FAGCN~\cite{fagcn} proposes self-gating mechanism to adaptively learn more information beyond low-frequency information in GNNs. 

\textbf{Graph Transformer.} Graph Transformer (GT) has attracted considerable attentions on long-range interaction. GT~\cite{gt} proposes to employ Laplacian eigenvectors as PE with randomly flipping their signs. Graphormer~\cite{graphormer} takes the distance of the shortest path between two nodes as spatial encoding, which is involved in attention calculation as a bias. GraphGPS~\cite{gps} provides different choices for PE, consisting of LapPE, RWSE, SignNet and EquivStableLapPE. SGFormer~\cite{SGFormer} is empowered by a simple attention model that can efficiently propagate information among arbitrary nodes. Recently, ~\citet{xing2024less} are the first to reveal the over-globalizing problem in graph transformer, and propose CoBFormer to improve the GT capacity on local modeling with a theoretical guarantee. 

\end{document}